\documentclass[11pt]{article}

\usepackage{times}
\usepackage{graphicx}
\usepackage{amsmath,amssymb}
\usepackage{booktabs}
\usepackage{hyperref}
\usepackage{microtype}
\usepackage{enumitem}
\usepackage{amsthm}

\usepackage[utf8]{inputenc} % 编码支持
\usepackage{amsmath, amssymb} % 数学公式
\usepackage{graphicx} % 插图支持
\usepackage{natbib} % 引用
\usepackage{url} % 超链接支持
\usepackage{hyperref}
\hypersetup{
	colorlinks=true,
	linkcolor=blue,
	filecolor=blue,      
	urlcolor=blue,
	citecolor=blue,
}

\usepackage{algorithm}
\usepackage{algpseudocode}
\usepackage{microtype}  % NeurIPS 模板推荐，用于排版优化

\newtheorem{definition}{Definition}
\newtheorem{condition}{Condition}
\newtheorem{theorem}{Theorem}
\newtheorem{lemma}{Lemma}

\usepackage[margin=1in]{geometry}
\title{Hybrid Combinatorial Multi-armed Bandits with Probabilistically Triggered Arms}

\author{
Kongchang Zhou \\
Southern University of Science and Technology \\
\texttt{12112825@mail.sustech.edu.cn}
\and
Tingyu Zhang \\
Southern University of Science and Technology  \\
\texttt{j74638239@gmail.com}
\and
Wei Chen \\
Microsoft Research\\
\texttt{weic@microsoft.com}
\and
Fang Kong \\
Southern University of Science and Technology \\
\texttt{kongf@sustech.edu.cn}
}

\date{}

\begin{document}
\maketitle

% ---------- Main Sections ----------
\begin{abstract}
The problem of combinatorial multi-armed bandits with probabilistically triggered arms (CMAB-T) has been extensively studied. Prior work primarily focuses on either the online setting where an agent learns about the unknown environment through iterative interactions, or the offline setting where a policy is learned solely from logged data. However, each of these paradigms has inherent limitations: online algorithms suffer from high interaction costs and slow adaptation, while offline methods are constrained by dataset quality and lack of exploration capabilities. To address these complementary weaknesses, we propose hybrid CMAB-T, a new framework that integrates offline data with online interaction in a principled manner. Our proposed hybrid CUCB algorithm leverages offline data to guide exploration and accelerate convergence, while strategically incorporating online interactions to mitigate the insufficient coverage or distributional bias of the offline dataset. We provide theoretical guarantees on the algorithm’s regret, demonstrating that hybrid CUCB significantly outperforms purely online approaches when high-quality offline data is available, and effectively corrects the bias inherent in offline-only methods when the data is limited or misaligned. Empirical results further demonstrate the consistent advantage of our algorithm. 
% We study the combinatorial multi-armed bandit problem with probabilistically triggered arms (CMAB-T) and investigate how unbiased offline data can improve online performance. We propose CUCB-p, an algorithm that treats past data as part of the early online rounds, enabling a form of early stopping. The regret analysis shows that CUCB-p achieves a standard regret bound minus a data-dependent term \( O\left(\sum_{i=1}^m \sqrt{N_i} \right) \), where \( N_i \) is the number of times arm \( i \) was triggered in the offline data. This result recovers the regret of Chen et al. (2017) as a special case, and further shows that with sufficiently large offline data, CUCB-p can achieve constant regret. Our findings highlight the potential of properly incorporating unbiased offline data to accelerate learning in CMAB-T settings.
\end{abstract}

\section{Introduction}

Combinatorial multi-armed bandits with probabilistically triggered arms (CMAB-T) provide a powerful framework for modeling a broad class of real-world sequential decision-making problems, including influence maximization, learning to rank, and large language model cache \citep{chen2013combinatorial,chen2016combinatorial,wang2017cmabt,wen2017online,kong2023onlineDC,liu2023contextualCMAB,liu2025offlineCMAB,pope2022scalingtransformers,zhu2023optimalcaching,gim2023promptcache,qu2024mobilellm}. In these settings, a decision-maker repeatedly selects a combinatorial action, typically a subset of base arms, and receives partial feedback governed by a probabilistic triggering process.

% The majority of existing work on CMAB-T has concentrated on the \emph{online setting}, where an agent learns through trial-and-error by interacting with the environment over multiple rounds. While this approach allows for adaptive learning and active exploration, it often incurs high feedback collection costs and suffers from slow convergence, particularly in large-scale or high-stakes domains. 
% At the other end of the spectrum, \emph{offline learning} methods aim to learn decision policies solely from pre-collected data logs. These approaches avoid the costs of online interaction but are limited by challenges such as sensitivity to the quality and coverage of logged data, and the inability to explore actively. Such limitations can significantly hinder their effectiveness, especially when dealing with sparse or biased datasets.

Most existing work on CMAB-T has focused on the \emph{online setting}, where an agent learns through trial and error by interacting with the environment over multiple rounds \citep{chen2013combinatorial,chen2016combinatorial,wang2017cmabt,wen2017online,kong2023onlineDC,liu2023contextualCMAB,liu2024combinatorialMDP}. While this approach enables adaptive learning and active exploration, it often incurs high feedback collection costs and suffers from slow convergence—particularly in large-scale or high-stakes domains. 

A study~\citep{liu2025offlineCMAB} has begun to explore the \emph{offline setting} for CMAB-T, where the goal is to learn decision policies from pre-collected data logs, thereby avoiding the expense of online interaction. However, offline learning is highly sensitive to the quality and coverage of the logged data. For example, rare but important action combinations may be missing, and distributional shifts between the offline data set and online environment can lead to suboptimal performance. Moreover, the lack of active exploration limits the learner’s ability to gather information about underexplored or high-uncertainty actions.

The limitations of purely online or offline learning motivate the study of hybrid learning methods, which use offline data to warm-start online learning \citep{shivaswamy2012multiHistory,song2023hybrid,oetomo2023cutting,agnihotri2024online,cheung2024minucb,qu2025hybrid}. These approaches balance the cost-free nature of offline data with the adaptability of online exploration, often leading to improved sample efficiency in practice.
While hybrid methods have been studied in the classical MAB problems, their extension to the general CMAB-T setting remains largely unexplored.

The technical challenge arises when incorporating the offline data into the regret analysis of the online CMAB-T. In particular, we must determine \emph{when} to rely on the pure online observation and \emph{when} the offline data (may be biased) is sufficiently reliable to be used. In the MAB setting, the regret admits a clean decomposition: it can be expressed as the sum over arms of the number of times each suboptimal arm is pulled, multiplied by its corresponding sub-optimality gap. This makes it straightforward to quantify how offline data reduces regret by decreasing the selection count of suboptimal arms \citep{cheung2024minucb}. But in our considered CMAB-T setting, such gap-based reasoning is no longer directly applicable, where the per-round regret cannot be attributed to individual arms through simple suboptimality gaps. The regret depends on the triggered arms and the combinatorial reward structure, making it much more difficult to define a universal threshold for determining when to use offline data.

To overcome these challenges, our work focuses on the following fundamental questions:

\textit{
(1) How to derive an algorithm that effectively leverages offline data in the online CMAB-T setting?
}

\textit{
(2) Can we provide the corresponding theoretical guarantees that offline data leads to measurable improvement compared with purely online algorithms?
}

We answer these questions through the following contributions: 

\textbf{Problem Formulation.} 
We formally define the \emph{hybrid CMAB-T} (H-CMAB-T) setting by extending the classical CMAB-T framework to incorporate offline data. In particular, we define the offline dataset as a collection of observations over base arms, and introduce a notion of \emph{bias} based on the discrepancy between the offline and online mean rewards of each arm. 
This formulation provides a principled basis for assessing when offline data can be beneficial to online learning.

\textbf{Algorithm Design.} 
We propose a new algorithm \emph{hybrid CUCB} leveraging the biased offline data to improve the classic CUCB algorithm. This algorithm balances offline and online feedback through a dual-UCB mechanism. Specifically, we construct two confidence bounds for each base arm: one purely based on the feedback collected online, and another that hybridizes observations from both the offline data set and online interactions with an explicit bias correction. By selecting the minimum of the two UCB estimates, the algorithm adaptively leverages the offline data based on its quality. 

\textbf{Theoretical Analysis.} 
To overcome challenge from the core difference between MAB and CMAB-T, we draw on the intuition that while the bias may appear at the level of individual arms, the regret in CMAB-T arises from actions that involve multiple arms and triggering mechanisms. Motivated by this, we explore a connection between per-arm bias and action-level regret by considering a hypothetical allocation of the regret to the arms that could be triggered in each round. This perspective allows us to bridge the arm-level discrepancy introduced by offline data and the combinatorial nature of regret in CMAB-T. Leveraging this connection, we construct a threshold condition that determines whether the offline estimates are reliable enough to be used. Finally, We provide both gap-dependent and gap-independent regret bounds. Our results show that the algorithm achieves improved regret over standard online methods~\citep{wang2017cmabt}, with a provable \emph{saving term} that depends on the informativeness and reliability of the offline data. 
Our result recovers the standard online regret when offline data is absent or adversarial, and it matches or improves upon the results of~\citet{cheung2024minucb} when the problem reduces to classical MAB. We also derive lower bound to prove its tightness in Appendix~E.

\textbf{Empirical Evaluation}
We complement our theoretical analysis with empirical evaluations. The results consistently demonstrate that hybrid CUCB outperforms both purely online and purely offline baselines, highlighting its adaptability and robustness across varying data conditions.

\section{Related Work}

\textbf{Online Bandits.}
MAB problems have been extensively studied as a foundational model for sequential decision-making under uncertainty~\citep{auer2002finite,bubeck2012regret,lattimore2020bandit}. The combinatorial multi-armed bandit (CMAB) framework~\citep{chen2013combinatorial} generalizes classical MAB by allowing the learner to select subsets of arms (super arms) in each round, leading to richer modeling power and broader applicability. In particular, the CMAB with probabilistically triggered arms (CMAB-T) framework introduced by~\citet{chen2016combinatorial,wang2017cmabt} captures the settings such as influence maximization, online learning to rank where the reward depends not only on the chosen super arm but also on a random triggering process. 
This framework has also been extended to incorporate contextual information \citep{liu2023contextualCMAB}.
A line of work has established algorithms with theoretical regret guarantees under structural assumptions such as monotonicity and bounded smoothness~\citep{chen2016combinatorial, wang2017cmabt,wen2017online,liu2022batch,liu2023contextualCMAB,liu2024combinatorialMDP}. All these approaches operate entirely in the online setting.

\textbf{Offline Bandits.}
Offline learning in bandit and reinforcement learning has gained increasing attention due to the high cost of online exploration and the availability of logged historical data. It has been explored in many bandits settings like the classical MAB~\citep{rashidinejad2021bridging}, contextual MAB~\citep{rashidinejad2021bridging,jin2021pessimism,li2022pessimism} and neural contextual bandits~\citep{nguyen2021sample,nguyen2022offline}.
For combinatorial bandits,~\citet{liu2025offlineCMAB} recently propose CLCB, the first general framework for offline learning in CMAB problems, which characterizes dataset quality through coverage conditions, and provide near-optimal theoretical guarantees. 

\textbf{Hybrid Bandits.}
To mitigate the limitations of purely online or offline learning, hybrid methods aim to combine their respective advantages by using offline data to initialize or guide online exploration.  
Hybrid learning has been studied in various domains, including bandit problems~\citep{shivaswamy2012multiHistory,oetomo2023cutting,agnihotri2024online} and reinforcement learning~\citep{song2023hybrid, qu2025hybrid}.
Most of these hybrid methods assume that offline data is unbiased and directly compatible with the online environment~\citep{shivaswamy2012multiHistory,song2023hybrid,oetomo2023cutting,agnihotri2024online}. \citet{qu2025hybrid} assume a strongly biased offline dataset with a lower bound on the discrepancy between offline and online means. \citet{cheung2024minucb} do not require such assumptions and propose an algorithm that adaptively incorporates offline data based on its reliability. To the best of our knowledge, the hybrid learning problem in CMAB-T remains open.

\section{Problem Setup}
We first introduce the \textit{ hybrid combinatorial mutlti-armed bandits  with probabilistically triggered arms} (H-CMAB-T) problem. The H-CMAB-T problem explored in this paper is built upon the standard CMAB-T framework \citep{wang2017cmabt}. We begin by reviewing the classical CMAB-T setting, and then introduce how offline data is incorporated in our extension.

The online environment consists of $m$ base arms, represented as random variables $X_1, X_2, \ldots, X_m$, jointly distributed according to an unknown distribution $D^\text{on} \in \mathcal{D}$, where $D^\text{on}$ is supported on $[0,1]^m$ and $\mathcal{D}$ is the distribution family.
For each base arm $i \in [m]$, let $\mu_i^\text{on} = \mathbb{E}_{X \sim D^\text{on}}[X_i]$ denote its expected value, and define the vector $\mu^\text{on} = (\mu_1^\text{on}, \ldots, \mu_m^\text{on}) \in [0,1]^m$ as the mean vector of all arms. Note that $\mu^\text{on}$ is determined by the underlying distribution $D^\text{on}$. The learning process unfolds over discrete rounds $t = 1, 2, \ldots, T$. In each round:

\textit{1.} The learner selects a combinatorial action $S_t \in \mathcal{S}$ based on the previous rounds observation and feedback, where $\mathcal{S}$ is a predefined action space, possibly subject to structural constraints. The combinatorial action $S_t$ is also called ``super arm'' and in many cases it is a subset of base arms. 

\textit{2.} The environment draws an independent sample $X^{(t)} = (X_1^{(t)}, \ldots, X_m^{(t)}) \sim D^{\text{on}}$. 

\textit{3.} Playing action $S_t$  in the environment induces a random subset $\tau_t \subseteq [m]$ of arms to be triggered. The triggering process is stochastic: even given the environment outcome $X^{(t)}$ and the chosen action $S_t$, the triggered set $\tau_t \subseteq [m]$ may still exhibit randomness. We model this using a \textit{probability triggering function} $D^{\text{trig}}(S, X)$, which defines a distribution over subsets of $[m]$ conditioned on action $S$ and environment realization $X$. Formally, we assume that for each round $t$, the triggered set $\tau_t$ is independently drawn from $D^{\text{trig}}(S_t, X^{(t)})$, i.e., $\tau_t \sim D^{\text{trig}}(S_t, X^{(t)})$. Moreover, to enable algorithms to estimate $\mu_i^{\text{on}}$ from observed samples during online learning, we make the following identifiability assumption: the outcome of each arm $i$ does not depend on whether it is triggered. That is, $\mathbb{E}_{X \sim D^{\text{on}}, \tau \sim D^{\text{trig}}(S, X)}[X_i \mid i \in \tau] = \mathbb{E}_{X \sim D^{\text{on}}}[X_i] = \mu_i^{\text{on}},~\forall ~i\in[m]$.

\textit{4.} A non-negative reward $R(S_t, X^{(t)}, \tau_t) \in \mathbb{R}_{\ge 0}$ is revealed to the learner, which is a deterministic function of the chosen action $S_t$, the sampled instance $ X^{(t)}$, and the triggered set $\tau_t$. The expected reward of an action $S \in \mathcal{S}$ is given by $r_S(\mu) := \mathbb{E}[R(S, X, \tau)]$, where the expectation is taken over $X \sim D$ and $\tau \sim D^{\text{trig}}(S, X)$. We emphasize that $r_S(\mu)$ is a function of $S$ and the mean vector $\mu$.

% \smallskip
% \noindent
% \textit{Remark.} For simplicity of exposition, we assume that each super arm $S$ can be viewed as a subset of base arms. While in some applications the action space does not directly correspond to subsets of base arms, this assumption can often be made without loss of generality via appropriate problem reduction.
% For example, the \emph{online influence maximization} (OIM) problem can be modeled as a CMAB-T instance, as demonstrated by Chen et al.~\cite{wang2017cmabt}. In this formulation, the base arms correspond to edges in the network, each with a mean activation probability, while the action selects a set of $k$ seed nodes. Although the super arm is originally defined over nodes, it can be equivalently transformed into the set of outgoing edges from the selected nodes. Under this transformation, the super arm becomes a subset of base arms, aligning with our assumption.

The goal of the learner is to maximize the total expected reward over $T$ rounds, i.e., to design a learning algorithm that selects $S_1, \ldots, S_T$ to maximize $\sum_{t=1}^T \mathbb{E}[R(S_t, X^{(t)}, \tau_t)]$.

While the classical CMAB-T framework captures the core structure of combinatorial bandit problems with triggering, it assumes that all learning happens online from scratch. In many practical scenarios, however, a significant amount of data is already available prior to online interaction---collected from historical logs or prior deployments. For example, in \emph{online influence maximization} problem, the organizations often have access to past propagation traces---records of how information spread---which can serve as valuable offline data to accelerate online learning in new deployment scenarios. 

% For example, in the context of \emph{Online Influence Maximization (OIM)}, where the goal is to select a subset of users to maximize information spread in a social network, the environment may evolve over time. For instance, the influence probabilities on network edges may drift due to shifts in user behavior, platform updates, or external events. Although the current network may differ from before, organizations often have access to past propagation traces---records of how information spread under earlier conditions---which can serve as valuable offline data to guide current decision-making.
% Similarly, in learning-to-rank systems such as web search or product recommendation, large-scale user interaction logs (e.g., clicks, views, purchases) are typically available. Although these logs are generated under prior ranking models or in different contexts, they provide valuable offline observations that can accelerate online learning in new deployment scenarios. 

Motivated by this, we consider an extension of CMAB-T that incorporates such \emph{offline data}, and investigate how it can be used to improve learning performance. More specifically, the key difference between H-CMAB-T and CMAB-T problem is that before online learning, the player is given an offline data collection $\mathcal{B}$. 
It is worth noting that there may be discrepancies between offline data and the online environment. For example, in the OIM problem, due to the characteristics of the product or shifts in user preferences, the diffusion dynamics within social networks can differ. To characterize such phenomenon and avoid misleading of offline data, we consider that the arms in the offline data set and the online setting may have different means. Specifically, 
% For simplicity, we assume that the environment that the offline data collected (we call it offline environment) and the online learning environment are similar. That is, both the distribution of the offline arms $D^{\text{off}}$ and online arms $D^{\text{on}}$ belong to the same distribution family $\mathcal{D}$ (a standard assumption in bandit problems is that the arm distributions are sub-Gaussian)
% but possibly with different means. In particular, 
the outcomes of $m$ base arms in the offline data set can be represented as random variables $Y_1, Y_2, \ldots, Y_m$, jointly distributed according to an unknown distribution $D^{\text{off}}$ and the mean vector of the offline data is  $\mu^{\text{off}} = (\mu^{\text{off}}_1, \ldots, \mu^{\text{off}}_m)$. It is natural that $|\mu^{\text{on}}_i-\mu^{\text{off}}_i|\ge0$, and equality holds if and only if the offline data is unbiased. Without loss of generality, we denote $N_i$ as the number of the independent observations of arm $i$. Then the offline data set can be represented as $
\mathcal{B}:=\{N_i,\{Y_{i,s}\}_{s=1}^{N_i}\}_{i=1}^{m}$. 

\textbf{Bias control.} 
Besides, to quantify this discrepancy, we adopt the bias control vector $V = (V_1, \ldots, V_m)$ as a hyper-parameter which upper bounds the difference between the offline and online means for each arm:
\[
|\mu_i^{\text{off}} - \mu_i^{\text{on}}| \leq V_i, \quad \forall i \in [m].
\]
Since both means lie in $[0,1]$, we assume $V_i \in [0,1]$ for all $i$. Smaller values of $V_i$ indicate higher alignment between offline and online environments. In settings with prior knowledge—e.g., similar user populations or stable network dynamics—we may set $V_i$ to be small. In fully agnostic cases where no such knowledge is available, we conservatively set $V_i = 1$.
% We remark that the learner does not know the specific distribution $D^\text{on}$ or $D^\text{off}$. Finally, we do not make strong assumptions about the source of the offline data. It could originate from historical logs of a previous CMAB-T instance after several rounds of interaction, or it could be obtained directly by sampling from the underlying distribution in some way. For this reason, we deliberately use the term \emph{observation} to describe each entry in the offline data. This terminology reflects the fact that a data point may not necessarily correspond to an arm being actively selected or triggered during a previous learning process. It is intended to be general enough to encompass both logged and independently collected data.

\textit{Remark 1. As rigorously shown in Section 3 of \cite{cheung2024minucb}, in the presence of biased offline data, no hybrid algorithm in MAB can be guaranteed to outperform a purely online baseline unless some prior knowledge about the bias is available. This theorem highlights that incorporating some form of prior understanding of the bias is not just helpful but fundamentally necessary. To understand this challenge, one can consider the unknown $V$ setting and try to design a hybrid algorithm that learns $V$ during the online interaction. This raises a challenging trade-off: if $V$ is small, estimating it accurately may require excessive online samples, outweighing the benefit of offline data; if $V$ is large, offline estimates are often too biased to be useful, making a pure online strategy preferable. Exploring the unknown $V$ setting is valuable but technically demanding, and we leave it as an important direction for future work. } 

% This is also intuitive: if we know nothing about how a batch of offline data was collected, we cannot tell whether it is helpful or detrimental or even adversarial to the online learning objective.

% \textit{} 

Consequently, based on the above problem formulation, we define an H-CMAB-T instance as a tuple \(([m], \mathcal{S}, \mathcal{D}, D^{\text{trig}}, R, \mathcal{B})\). To make the learning problem well-defined and practically solvable, it remains to specify how actions are selected given current estimates of the arm statistics.
In many CMAB-T instances, the action space is exponentially large and the underlying optimization problem of selecting the optimal super arm is NP-hard~\citep{chen2013combinatorial,chen2016combinatorial}.
% For example, in applications such as influence maximization or constrained recommendation, even evaluating the best action given full knowledge of the arm means is computationally intractable.
To decouple the statistical estimation from the combinatorial optimization, prior works commonly assume the access to an \emph{offline oracle} that returns an approximate solution. This allows the learning algorithm to focus on estimating arm statistics while relying on the oracle to select actions.

\paragraph{Offline \texorpdfstring{$(\alpha, \beta)$}--approximation oracle $\mathcal{O}$.}
We assume access to an offline $(\alpha, \beta)$-approximation oracle, denoted by $\mathcal{O}$. This oracle takes as input the mean vector $\mu = (\mu_1, \ldots, \mu_m)$ and returns an action $S^{\mathcal{O}} \in \mathcal{S}$ such that $\mathbb{P} \left[ r_{S^{\mathcal{O}}}(\mu) \ge \alpha \cdot \operatorname{opt}_\mu \right] \ge \beta$, where $\alpha \in (0,1]$ is the approximation ratio, and $\beta \in (0,1]$ is the success probability. Here, $\operatorname{opt}_\mu$ denotes the optimal expected reward under mean vector $\mu$, defined as
% \begin{equation*}
%     \operatorname{opt}_\mu := \sup_{S \in \mathcal{S}} r_S(\mu),
% \end{equation*}
\(
\operatorname{opt}_\mu := \sup_{S \in \mathcal{S}} r_S(\mu).
\)
And we assume that $\operatorname{opt}_\mu$ is bounded for all $\mu$.

% This oracle captures the idea that given accurate estimates of arm means, we can (with probability at least $\beta$) find an approximately optimal action that achieves at least an $\alpha$ fraction of the optimal reward. 
% It shifts the algorithmic focus to the estimation of $\mu$, while the oracle handles the action selection.

Further, the objective of the learner is to minimize the $(\alpha, \beta)$--approximation regret defined as below~\citep{chen2013combinatorial, chen2016combinatorial,wang2017cmabt,wen2017online}.

\begin{definition}[{\texorpdfstring{$(\alpha, \beta)$}--approximation regret.}]
\label{def:approximation regret}
The $(\alpha, \beta)$-approximation regret of a learning algorithm $\mathcal{A}$ over $T$ rounds under an H-CMAB-T instance \(([m], \mathcal{S}, \mathcal{D}, D^{\text{trig}}, R, \mathcal{B})\) is
\begin{equation*}
    \operatorname{Reg}^{\mathcal{A}}_{\mu^{\text{on}} ,\alpha,\beta}(T) := \alpha \cdot \beta \cdot T \cdot \operatorname{opt}_{\mu^{\text{on}}}  - \mathbb{E} \left[ \sum_{t=1}^T R(S_t^{\mathcal{A}}, X^{(t)}, \tau_t) \right] = \alpha \cdot \beta \cdot T \cdot \operatorname{opt}_{\mu^{\text{on}}}   - \mathbb{E} \left[ \sum_{t=1}^T r_{S_t^{\mathcal{A}}}(\mu^{\text{on}} ) \right],
\end{equation*}
where $S_t^{\mathcal{A}}$ is the action selected by algorithm $\mathcal{A}$ at round $t$, and the expectation is taken over the randomness of the environment outcomes $\{X^{(t)}\}_{t=1}^T$, the triggered sets $\{\tau_t\}_{t=1}^T$, and the internal randomness of the algorithm.
\end{definition}

This notion of regret captures how far the cumulative reward falls short of what could be obtained by always playing a near-optimal action provided by the oracle.

We now introduce several conditions that are used to establish regret guarantees. These conditions are widely adopted in the CMAB literature \citep{chen2016combinatorial,wang2017cmabt,wen2017online,liu2023contextualCMAB,liu2025offlineCMAB}. To facilitate the presentation, we denote $p_i^{D,S}$ as the probability that arm $i$ is triggered when action $S$ is selected in environment $D$.

\begin{condition}[Monotonicity] 
\label{cond:mono}
We say that a CMAB-T problem instance satisfies monotonicity, if for any action $S \in \mathcal{S}$, for any two distributions $D, D' \in \mathcal{D}$ with expectation vectors $\boldsymbol{\mu} = (\mu_1, \ldots, \mu_m)$ and $\boldsymbol{\mu}' = (\mu_1', \ldots, \mu_m')$, we have $r_S(\boldsymbol{\mu}) \leq r_S(\boldsymbol{\mu}')$ if $\mu_i \leq \mu_i'$ for all $i \in [m]$.
\end{condition}

\begin{condition}[1-Norm TPM Bounded Smoothness]
\label{cond:TPM}
We say that a CMAB-T problem instance satisfies 1-norm TPM bounded smoothness, if there exists $B \in \mathbb{R}^+$ (referred as the bounded smoothness constant) such that, for any two distributions $D, D' \in \mathcal{D}$ with expectation vectors $\boldsymbol{\mu}$ and $\boldsymbol{\mu}'$, and any action $S$, we have
\(|r_S(\boldsymbol{\mu}) - r_S(\boldsymbol{\mu}')| \leq B \sum_{i \in [m]} p_i^{D,S} |\mu_i - \mu_i'|.  \)
\end{condition}
%\wei{$p_i^{D,S}$ seems not defined.}

% \wei{The explanation about the CMAB-T framework is long. If we need space, we can see if we could simplify it, since the main online part of the framework is the same as in \citep{wang2017cmabt}, so
% 	we do not need more explanations or motivations on certain aspects. For example, the explanation below about the two conditions could be significantly shortened.}
	
% The intuitions behind the two reward function conditions are as follows:

% \begin{itemize}
%     \item \textbf{Condition 1} is motivated by a natural monotonicity intuition: if the mean value of each arm in one set is higher than that in another set, then we naturally expect any action $S$ to yield a higher reward under the first set.
    
%     \item \textbf{Condition 2} reflects the idea that, in the CMAB-T setting, the triggering probabilities of arms should influence how much they contribute to the overall reward. For example, if an arm has a very small triggering probability, its contribution to the overall reward would also be small, and thus we may tolerate more uncertainty in its mean estimate. 
%     Conversely, if an arm is frequently triggered, its contribution to the reward would be high and we desire more accurate estimation of its mean. 
% \end{itemize}

% We remark that these conditions are intrinsic to the problem instance itself and are not specific to whether the setting involves online or offline learning.

The two reward function conditions encode natural intuitions in the CMAB-T setting: Condition~\ref{cond:mono} reflects monotonicity—if all arm means are higher in one set than another, any action should yield a higher expected reward; Condition~\ref{cond:TPM} captures the role of triggering probabilities—arms that are triggered more often contribute more to the reward and thus require more accurate mean estimates, while less frequently triggered arms can tolerate greater uncertainty.

%\wei{Not sure if we want to say that these conditions are inherent to the online environment. They are inherent to the problem setting. They talk about the properties of the problem instance, and not directly related whether the setting is online learning or offline learning.}

% Note that the aforementioned conditions are stated with respect to arbitrary distributions, and hence remain valid when offline data is incorporated---particularly in the case of \emph{unbiased} offline data, whether the conditions apply to the offline data itself or to the relationship between offline and online data.

\section{The Hybrid CUCB Algorithm}

\begin{algorithm}
\caption{Hybrid CUCB with Computation Oracle}
\label{alg:hybrid cucb}
\begin{algorithmic}[1]
\Require Valid bias bound $V$, number of arms $m$, offline data 
$\mathcal{B}:= \{N_i, \{Y_{i,s}\}_{s=1}^{N_i} \}_{i=1}^{m}$, horizon $T$, Oracle

\For{each arm $i \in [m]$}
    \State $\hat{\mu}_i^{\text{off}} \leftarrow \frac{1}{N_i} \sum_{s=1}^{N_i} Y_{i,s}, \quad T_i \leftarrow 0$, \quad $\hat{\mu}_i^{\text{on}} \leftarrow 0$ \hfill
\EndFor

\For{$t = 1,2,\ldots,T$}
    \For{each arm $i \in [m]$}
        \State $\mathrm{rad}_t(i) \leftarrow \sqrt{ \frac{2 \log(4mt^3)}{T_i} }$ \label{alg:line:radius:online}  \hfill\Comment{$=\infty$ if $T_i = 0$}
        \State $\mathrm{rad}_t^{\mathrm{S}}(i) \leftarrow \sqrt{ \frac{2 \log(4mt^3)}{N_i + T_i} } + \frac{N_i}{N_i + T_i} V_i$  \label{alg:line:radius:hybird}\hfill\Comment{$=\infty$ if $N_i + T_i = 0$}
        \State $\mathrm{UCB}_t(i) \leftarrow \hat{\mu}_i^{\text{on}} + \mathrm{rad}_t(i)$ \label{alg:line:ucb:online}
        \State $\mathrm{UCB}_t^{\mathrm{S}}(i) \leftarrow \frac{N_i \hat{\mu}_i^{\text{off}} + T_i \hat{\mu}_i^{\text{on}}}{N_i + T_i} + \mathrm{rad}_t^{\mathrm{S}}(i)$ \label{alg:line:ucb:hybrid}
        \State $\bar{\mu}_i \leftarrow \min \left\{ \mathrm{UCB}_t(i),\, \mathrm{UCB}_t^{\mathrm{S}}(i),\, 1 \right\}$
    \EndFor

    \State $S \leftarrow \text{Oracle}(\bar{\mu}_1, \ldots, \bar{\mu}_m)$
    \State Play action $S$, triggering a set $\tau \subseteq [m]$ of base arms
    \For{each $i \in \tau$ with feedback $X_i^{(t)}$}
        \State \quad $T_i \leftarrow T_i + 1$\quad $\hat{\mu}_i^{\text{on}} \leftarrow \hat{\mu}_i^{\text{on}} + (X_i^{(t)} - \hat{\mu}_i^{\text{on}}  )/T_i$
    \EndFor
\EndFor
\end{algorithmic}
\end{algorithm}

In this section, we provide an algorithm, hybrid CUCB (Algorithm \ref{alg:hybrid cucb}), aiming to leverage \emph{useful} offline data to accelerate the online learning efficiency. 
The hybrid CUCB algorithm runs as follows. In each round, the algorithm computes two UCB vectors:
\[
\mathrm{UCB}_t = (\mathrm{UCB}_t(1), \ldots, \mathrm{UCB}_t(m)), \quad
\mathrm{UCB}^\mathrm{S}_t = (\mathrm{UCB}^\mathrm{S}_t(1), \ldots, \mathrm{UCB}^\mathrm{S}_t(m)),
\]
and then feeds the coordinate-wise minimum two of them 
% \[
% \min\left\{ \mathrm{UCB}_t, \mathrm{UCB}^\mathrm{S}_t \right\}
% \]
into the $(\alpha, \beta)$-approximation oracle to select an action.

The vector $\mathrm{UCB}_t$ follows the standard CUCB construction~\citep{wang2017cmabt} (Line \ref{alg:line:radius:online} and \ref{alg:line:ucb:online}), representing the conventional UCB established with the pure online feedback, where $T_i$ denotes the number of times that arm $i$ has been triggered. 
% In contrast, $\mathrm{UCB}^\mathrm{S}_t$ is inspired by the hybrid structure proposed in~\citet{cheung2024minucb}, which integrates both offline and online data under the MAB problem setting.

As to H-CMAB-T problem, to effectively leverage offline data while remaining robust to distributional mismatch, we design a hybrid confidence bound $\mathrm{UCB}^\mathrm{S}_t$ that adaptively incorporates offline observations. Intuitively, when the offline mean of an arm is close to its online counterpart, the offline data should be more trusted. Conversely, if the discrepancy between the two is large, the algorithm should rely primarily on online feedback. 

% To quantify this discrepancy, we adopt the bias control vector $V = (V_1, \ldots, V_m)$ as a hyper-parameter which upper bounds the difference between the offline and online means for each arm:
% \wei{Perhaps the introduction of $V_i$ should be in the setup section. }
% \[
% |\mu_i^{\text{off}} - \mu_i^{\text{on}}| \leq V_i, \quad \forall i \in [m].
% \]
% Since both means lie in $[0,1]$, we assume $V_i \in [0,1]$ for all $i$. Smaller values of $V_i$ indicate higher alignment between offline and online environments. In settings with prior knowledge—e.g., similar user populations or stable network dynamics—we may set $V_i$ to be small. In fully agnostic cases where no such knowledge is available, we conservatively set $V_i = 1$.

Based on this intuition, we construct $\mathrm{UCB}^\mathrm{S}_t(i)$ using a weighted empirical mean and a bias-adjusted confidence radius (Line \ref{alg:line:radius:hybird} and \ref{alg:line:ucb:hybrid}). The empirical mean aggregates offline and online samples proportionally to their counts, while the confidence radius consists of two components: a standard deviation term based on the total offline and online sample size $N_i + T_i$, and a bias penalty scaled by the discrepancy bound $V_i$. The weight $N_i / (N_i + T_i)$ ensures that the penalty becomes more prominent as more offline data is used.

Finally, by taking the minimum between the two UCB estimates, the algorithm can exploit useful offline data. Intuitively, if $N_i$ is large and $V_i$ is small such that $\mathrm{UCB}^S_t(i) < \mathrm{UCB}_t(i)$, then the offline data is useful for online exploration and the algorithm utilizes the hybrid $\mathrm{UCB}^S_t(i)$. Otherwise, if $\mu_i^{\text{off}}$ and $\mu_i^{\text{on}}$ are far apart, then $\mathrm{UCB}^S_t(i)$ becomes large. The algorithm would default to $\mathrm{UCB}_t(i)$, effectively ignoring offline data. 
In both cases, the selection rule ensures that the decision is made conservatively, based on the estimated trustworthiness of the offline data. We next provide the regret upper bound for Algorithm \ref{alg:hybrid cucb} in Section \ref{sec:results}.

\section{Theoretical Analysis}\label{sec:results}

In this section, we provide the theoretical results for hybrid CUCB. We first provide the gap-dependent regret upper bound and the corresponding discussions. The gap-independent regret analysis comes later. The complete proof is provided in the appendix.

\subsection{Gap-Dependent Bound}\label{subsec:Gap-Dependent Bound}

We first define the reward gaps used in the regret analysis.

\begin{definition}[Gap~\citep{wang2017cmabt}]
\label{def:gap}
Fix a distribution $D$ and its expectation vector $\boldsymbol{\mu}$. For each action $S$, we define the gap $\Delta_S = \max(0, \alpha \cdot \text{opt}_{\boldsymbol{\mu}} - r_S(\boldsymbol{\mu}))$. 
For each arm $i$, we define
\[
\Delta_{\min}^i = \inf_{S \in \mathcal{S} : p_i^{D,S} > 0, \Delta_S > 0} \Delta_S, \quad 
\Delta_{\max}^i = \sup_{S \in \mathcal{S} : p_i^{D,S} > 0, \Delta_S > 0} \Delta_S.
\]
As a convention, if there is no action $S$ such that $p_i^{D,S} > 0$ and $\Delta_S > 0$, we define 
$\Delta_{\min}^i = +\infty$, $\Delta_{\max}^i = 0$. Further define $\Delta_{\min} = \min_{i \in [m]} \Delta_{\min}^i,~\Delta_{\max} = \max_{i \in [m]} \Delta_{\max}^i$. 
\end{definition}

Let $\tilde{S} = \{i \in [m] \mid p_i^{\boldsymbol{\mu},S} > 0\}$ be the set of arms that could be triggered by $S$. Let $K = \max_{S \in \mathcal{S}} |\tilde{S}|$. 
% \wei{In our ICML24 paper, ``Combinatorial Multivariant Multi-Armed Bandits with Applications to Episodic Reinforcement Learning and Beyond'', we used a better definition for $K$, as
% 	 $K = \max_{S \in \mathcal{S}} \sum_{i=1}^m p_i^{D,S}$. That is, $K$ is the maximum {\em expected number of base arms} that can be triggered by an action.
% 	 It could be much smaller than the old definition. But that paper uses this definition for a general setting of CMAB-MT. I am asking Xutong to give a regret analysis
% 	 using this new definition on CMAB-T. We do not need to change this definition, but we should remember to change it in the next version.}
% For convenience, we use $[x]_0$ to denote $\max\{\lfloor x \rfloor, 0\}$ for any real number $x$ \fang{do we use this notation in the paper?}. 
% To formally 
% we also introduce a discrepancy measure between offline environment and online learning environment define firstly introduced~\cite{cheung2024minucb} in below \fang{explain why we use this measure, not briefly mention it is introduced in xxx}:
To formally capture the influence of the discrepancy between offline and online environment, we introduce a measure $\omega_i:= V_i+\mu_i^\text{off}-\mu_i^\text{on},~i\in[m]$. 
By the definition of $V$, we have that $\omega_i \in [0, 2V_i]$. Intuitively, the quantity $\omega_i$ allows us to express how much the offline data for arm $i$ deviates from the true online behavior, and plays a key role in determining the extent to which the offline data influences the online learning. 

\begin{theorem}[Gap-Dependent Regret Bound] 
\label{thm:regret-bound}
For an H-CMAB-T problem \(([m], \mathcal{S}, \mathcal{D}, D^{\text{trig}}, R, \mathcal{B})\) that satisfies monotonicity (Condition~1) and TPM bounded smoothness (Condition~2), the hybrid CUCB algorithm with an input bias control vector $V$ and an $(\alpha, \beta)$-approximation oracle achieves an $(\alpha, \beta)$-approximate gap-dependent regret bounded by:
\begin{equation}
    \operatorname{Reg}_{\mu^\text{on},\alpha,\beta}(T) \leq \sum_{i \in [m]} \max\left\{ 
    \frac{64 \sqrt{2} B^2 K \log(4mT^3)}{\Delta_{\min}^i} - 
    8B \sqrt{2N_i' \log(4mT^3)}, 0 
    \right\} + 4Bm + \frac{\pi^2}{6}\Delta_{\max}, \label{eq:regret}
\end{equation}
where 
\[
N_i' = N_i \cdot \max\left\{ 1 - \frac{2BK \omega_i}{\Delta_{\min}^i}, 0 \right\}^2.
\]
\end{theorem}

Following Theorem~\ref{thm:regret-bound}, we now provide a detailed interpretation of the regret bound and its implications for how offline data is used by our algorithm. 

A key quantity in the bound is $N_i'$, which represents the amount of \emph{effectively utilized} offline data for arm $i$. 
The multiplicative factor can be interpreted as the \emph{utilization rate} of the offline data. For a fixed online learning setting, the term ${2BK}/{\Delta_{\min}^i}$ is constant, so the utilization rate increases as the discrepancy $\omega_i$ decreases. When the offline data is unbiased (i.e., $V_i = \omega_i = 0$), we have full utilization: $N_i' = N_i$. In contrast, when $\omega_i \ge \Delta_{\min}^i / (2BK)$, the utilization rate drops to zero, and the offline data is effectively ignored. This reflects our design intuition: offline data that closely matches the online environment should be trusted more and used more aggressively.
The result of Theorem 1 recovers the result of CMAB-T \citep{wang2017cmabt} as a special case when $N_i' = 0$ for all $i$. The setting may correspond to the case where the offline data do not exist (i.e. $N_i=0$ for all $i\in[m]$) or the case that the offline data is fully misaligned with the online environment. 

In general, our regret bound takes the form of the traditional regret in a purely online setting plus a benefit term of order $O(-\sqrt{N_i'})$. One might wonder why the adjustment is of order $O(-\sqrt{N_i'})$ instead of $O(-N_i')$ in \citet{cheung2024minucb}, which subtracts a term proportional to the effective number of plays, roughly $N'_i$, times the per-play regret.
This difference arises from the distinct analytical techniques used in the MAB and CMAB-T settings. 
In MAB, the regret can be directly decomposed by counting the number of times each sub-optimal arm is selected. Thus, the benefit from offline data is proportional to the number of these selections avoided.
 In contrast, the CMAB-T analysis—enabled by the monotonicity and TPM condition—bounds the regret by analyzing the discrepancy between the UCB estimates and the true mean rewards. 
 Intuitively, $O(-\sqrt{N'_i})$ comes from the regret saved in this discrepancy. With the offline data, we can interpret the online learning process as beginning from the $(N'_i+1)$-th observation for each arm $i$. The resulting saving in the discrepancy between the UCB estimates and the true mean rewards is approximately $\sum_{s=1}^{N_i'}\sqrt{\log (4mT^3)/s} = O(\sqrt{N'_i \log (4mT^3)})$. 
When $N'_i$ is larger than $64 B^2K^2 \log (4mT^3)/(\Delta_{\min}^i)^2$ , the regret incurred during the online phase becomes bounded by a constant independent of $T$. 
% \wei{To make the regret a constant independent of $T$, we also need $N'_i$ to be large enough for factors related to $K$ and $B$, so perhaps we need to include $K$ and $B$ explicitly. Moreover, saying that $N'_i$ is larger than $O(\log T/(\Delta_{\min}^i)^2)$ is a weird expression, since the $O(\cdot)$ notation is the asymptotic less than or equal to. Perhaps it is better just to write the actual inequality to remove the $\log T$ term in the regret.}
This aligns with the same observations in the reduced MAB setting discussed in \citet{cheung2024minucb}.

\subsection{Gap-Independent Bound}
We then analyze the gap-independent regret upper bound. We obtain two candidate bounds, denoted as $\psi$ and $\gamma$, each derived from a different proof technique. The final regret bound takes the minimum among them. 

\begin{theorem}[Gap-Independent Regret Bound] \label{thm:gap-indep}
For an H-CMAB-T problem \(([m], \mathcal{S}, \mathcal{D}, D^{\text{trig}}, R, \mathcal{B})\) that satisfies monotonicity (Condition~1) and TPM bounded smoothness (Condition~2), the hybrid CUCB algorithm with an input bias control vector $V$ and an $(\alpha, \beta)$-approximation oracle achieves an $(\alpha, \beta)$-approximate gap-independent regret bounded by:
\begin{equation}
    \operatorname{Reg}_{\mu^\text{on},\alpha,\beta}(T) \leq \min\{  \psi, \gamma \} +4Bm + \frac{\pi^2}{6}\Delta_{\max} \,,\label{eq:gap-independent regret}
\end{equation}
where $\psi$ and $\gamma$ are two candidate bounds derived via distinct proof techniques:
\begin{equation} \label{eq:psi}
\psi = 8\sqrt{2}B \sqrt{\log (4mT^3)} \left( \sum_{i \in [m]} \max \left\{ \sqrt{\frac{KT}{m}} - \sqrt{N_i''}, 0 \right\} + \sqrt{mKT} \right) \,,
\end{equation}

% \begin{equation} \label{eq:zeta}
% \zeta = 4\sqrt{2} BKT \sqrt{\log (4mT^3)} \sum_{i \in [m]} \frac{1}{\sqrt{N_i}}+2BKT \sum_{i\in [m]}\omega_i
% \end{equation}

\begin{equation} \label{eq:gamma}
\gamma = 16BKT\sqrt{\frac{2\log(4mT^3)}{\tau_*}}+BKT\omega_{\max} \,.
\end{equation}
Here
\begin{equation} \label{eq:N_i''}
    N_i'' = N_i \cdot \max\left\{ 1 - \frac{\omega_i}{4\sqrt{2}}\sqrt{\frac{KT}{m \log(4mT^3)}}, 0 \right\}^2\,,\quad \omega_{\max}=\max_i\omega_i\,,
\end{equation}
and $\tau_*$ is defined via
\[
\begin{aligned}
& \max_{\tau,\, n} \quad \tau \\
\text{s.t.} \quad & \tau \leq N_i + n(i) \text{ where } \tau \in \mathbb{N}, n(i)  \in \mathbb{N}, \forall i \,,\\
& \sum_{i \in [m]} n(i) \leq KT\,.\\
\end{aligned}
\]
\end{theorem}
These two upper bounds capture different aspects of how offline data can reduce exploration cost in the H-CMAB-T setting. We will interpret each bound, compare their relative strengths, and highlight how they recover or generalize existing results in the literature as follows.

% The first bound, $\psi$, is based on a trade-off between  good event and bad event serves as a robust lower guarantee on regret. 
Formally, the first bound $\psi$ involves the quantity $N_i''$, defined analogously to $N_i'$ in the gap-dependent bound, and it is interpreted as the amount of \emph{effectively used} offline data. Similarly, the quantity $N_i''$ embodies the guiding principle behind our algorithmic design in Section~\ref{subsec:Gap-Dependent Bound}: the more aligned the offline data is with the online environment, the more confidently and extensively it can be incorporated into the learning process.
The setting where $N_i'' = 0$ for all $i$ recovers the pure online CMAB-T problem in~\citep{wang2017cmabt}, and the resulting bound matches their gap-independent result in order. In this sense, $\psi$ generalizes their analysis by quantifying the potential reduction in regret due to informative offline data via an $O(-\sqrt{N_i''})$ saving term.
% It is worth noting that the use of a $\max\{\cdot, 0\}$ operator implies that $\psi$ lies between $8B \sqrt{mBKT\log(4mT^3)}$ (when offline data is maximally beneficial) and $16B \sqrt{mBKT\log(4mT^3)}$ (when offline data is entirely ignored). Thus, while $\psi$ captures meaningful offline data effects, it does not improve the regret rate order relative to pure online learning.
Moreover, it is worth noting that the use of the $\max\{\cdot, 0\}$ operator implies that $\psi$ ranges between a best-case value (when $N_i''$ is so large that the $\max\{\cdot, 0\}=0, ~\forall ~i$ ) and a worst-case value (when $N_i''=0,~ \forall ~i$) matching the pure online regret bound. Specifically, $\psi$ lies between $8B\sqrt{mKT \log (4mT^3)}$ and $16B\sqrt{mKT \log (4mT^3)}$, depending on the informativeness of the offline data. Therefore, although $\psi$ reflects meaningful offline benefits and can cut down half of the regret at the best case, it does not improve the regret order corresponding to the specific problem parameters.

The second bound, $\gamma$, is derived via a relaxation of exploration constraints into a covering linear program. The LP solution $\tau_*$ appearing in $\gamma$ satisfies a uniform lower bound $\tau_* \ge KT/m$, which ensures that the first term in $\gamma$ is always at most the worst case of $\psi$. It can still be smaller when $N_i$ is large and $w_{\max}$ is small. In some extreme cases where $w_{\max} \le 1/BKT$ and $N_i \ge (BKT)^2 \log (4mT^3)$ , the bound $\gamma$ tends to be of constant order which is independent of $T$, highlighting the potential for offline data to fully eliminate exploration cost under perfect alignment. 
Moreover, $\gamma$ structurally aligns with recent work on leveraging offline data in the classical MAB setting~\citep{cheung2024minucb}. By setting $K = B = 1$, our H-CMAB-T problem reduces to a hybrid MAB scenario. In this special case, $\gamma$ recovers (and slightly tightens) \citet{cheung2024minucb}: their bound includes a saving term of the form $2TV_{\max}$, whereas ours uses $Tw_{\max}$ with $w_{\max} \le 2V_{\max}$.

% Furthermore, the LP solution $\tau_*$ appearing in $\gamma$ satisfies a uniform lower bound $\tau_* \ge KT/m$, which ensures that $\gamma$ is always at most $\psi$, and can be significantly smaller when $N_i$ is large and $w_{\max}$ is small. In the extreme case where $w_{\max} = 0$ and $N_i \to \infty$, the bound $\gamma$ tends to zero, highlighting the potential for offline data to fully eliminate exploration cost under perfect alignment.

We now compare the two bounds in terms of tightness and interpretability. The bound $\psi$ provides a uniform guarantee and reflects a conservative lower baseline. While it never diverges, it also does not yield a tighter rate even when offline data is abundant.
In contrast, $\gamma$ can become substantially tighter in favorable regimes. When the offline data is highly informative (i.e., large $N_i$ and small $\omega_i$), $\gamma$ can reduce the regret significantly. For example, in the ideal case of $N_i \ge (BKT)^2 \log (4mT^3)$ and $\omega_{\max} \leq 1/BKT$, the bound tends to be a constant, matching our expectation that regret should vanish when offline information fully resolves arm uncertainty.

% Compared to $\zeta$, the bound $\gamma$ is structurally similar but less explicit: the impact of offline data is encoded via the LP solution $\tau_*$, which aggregates instance-level exploration requirements. While $\gamma$ is less interpretable than $\zeta$, it still provides a potentially tighter bound due to the optimization over all coverage patterns.

Together, these two bounds form a comprehensive characterization of the gap-independent regret in H-CMAB-T. They offer different trade-offs between robustness, interpretability, and tightness, and demonstrate how the size, bias, and coverage of offline data influence the learning performance

\section{Experiments}\label{exp}
In this section, we compare our proposed hybrid CUCB with existing CUCB for the pure online setting \citep{wang2017cmabt} and CLCB for the pure offline setting \citep{liu2025offlineCMAB}. To evaluate the performance of CLCB, we first use this algorithm to select an action based on the offline data set and always select this action in the following rounds. 
% We adopt the cascading bandits on the application
% of learning to rank as the corresponding CMAB problem. 
% More descriptions about the reward function and triggering mechanism can be found in Section 4.1 of \citet{liu2025offlineCMAB}. 
% For synthetic experiments part, we generate $10$ base arms and set the action size as $5$. All results are averaged over $20$ runs, and the error bar is defined as the standard deviation divided by $\sqrt{20}$. For real world applications, we use MovieLens as data set. 
 For simplicity, we assume that $N_i=N$ and $V_i=V$ for any arm $i$.
Due to the space limit, more details about the reward function and triggering mechanism, as well as the experimental setting and real-world validations, are deferred to appendix.
 % For simplicity, we assume that $N_i=N$ and $V_i=V$ for any arm $i$.

% We first consider the unbiased offline data set with $V_i=0$ for all $i\in[m]$ and evaluate performances under varying offline dataset sizes $N\in\{10,50,200\}$. The results are shown in Figure \ref{fig:exp1}. Our hybrid CUCB (abbreviated as H-CUCB) consistently outperforms the pure online CUCB algorithm across all offline data sizes. This improvement is attributed to the warm-start effect provided by the offline data, which allows hybrid CUCB to reduce initial exploration and make more informed decisions early on. As the amount of offline data increases, the advantage of the hybrid approach becomes more pronounced. When the offline dataset is sufficiently large, i.e., when $N=200$, our hybrid CUCB only occurs constant regret. Moreover, hybrid CUCB also demonstrates superior performance compared to the pure offline CLCB algorithm, particularly when the offline dataset is small. This is because CLCB relies entirely on offline estimates, which may be inaccurate when the data is insufficient, leading to suboptimal action choices.

We evaluate on unbiased offline datasets with varying sizes $N \in \{10,50,200\}$. 
As shown in Figure~\ref{fig:exp1}, hybrid CUCB consistently outperforms both 
online CUCB and offline CLCB. The improvement stems from the warm-start provided by 
offline data, which reduces early exploration. The advantage becomes more pronounced 
with larger $N$, and when $N$ is sufficiently large (e.g., $N=200$), hybrid CUCB achieves 
constant regret. Compared to CLCB, the hybrid approach is especially superior when 
offline data is scarce, since CLCB relies solely on potentially inaccurate offline estimates.

\begin{figure}[h]
    \centering
   \includegraphics[width=0.32\textwidth]{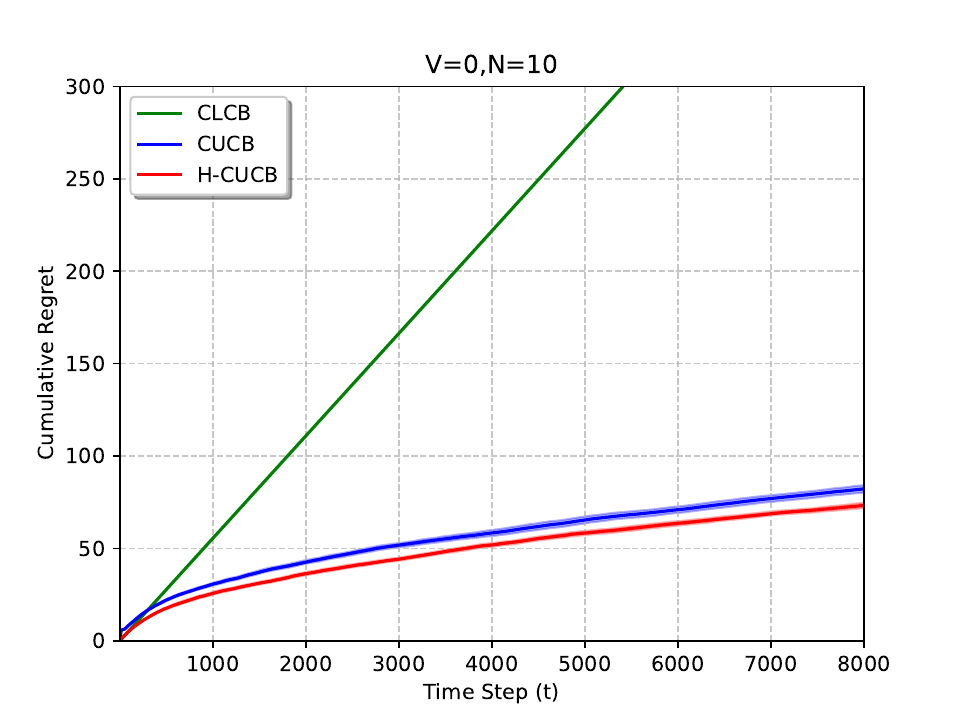}
   \includegraphics[width=0.32\textwidth]{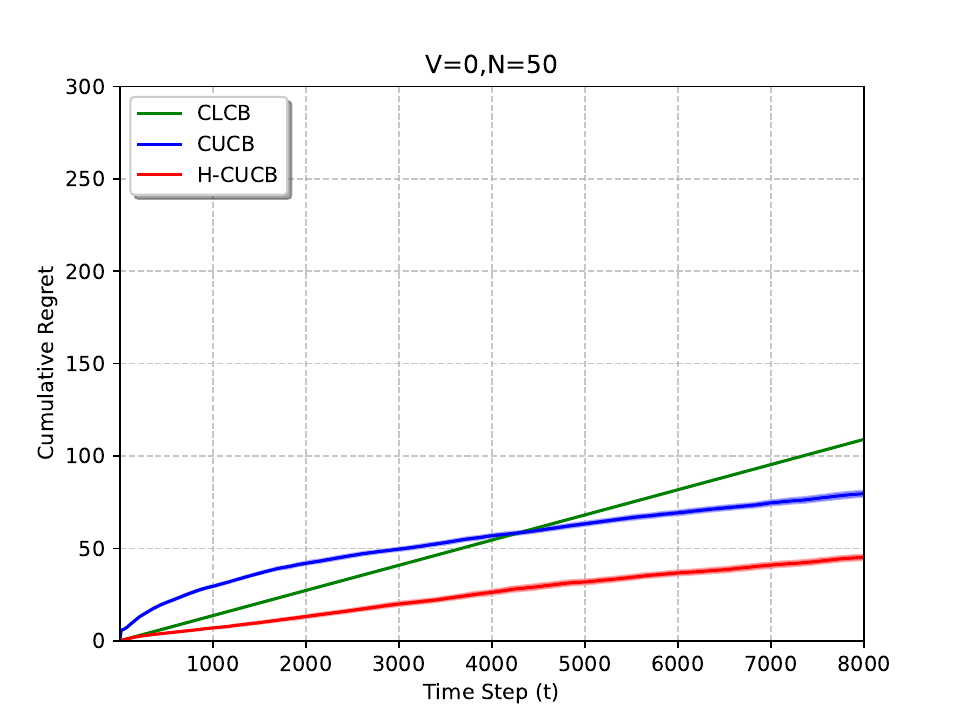}
   \includegraphics[width=0.32\textwidth]{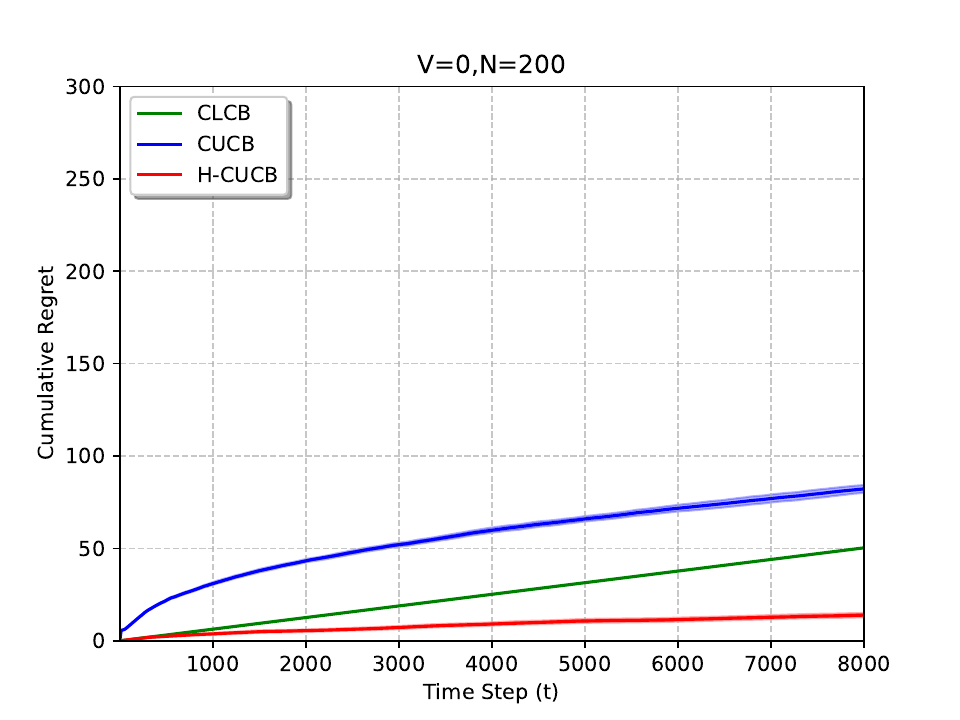}
    \caption{Performance comparison of hybrid CUCB against baselines with unbiased offline data set. }
    \label{fig:exp1}
\end{figure}

% We also test the robustness of algorithms when the distribution bias between the offline and online environments exists. We test different levels of bias $V\in\{0.2,0.3,0.4\}$ when the offline data size is sufficient $N=200$. The results are shown in Figure \ref{fig:exp_bias}. Our hybrid CUCB achieves the same or even better performance in all tested levels of bias.  

We further evaluate the robustness of the algorithms under distributional bias between the offline and online environments. Specifically, we consider varying levels of bias $V\in\{0.2,0.3,0.4\}$, assuming a sufficiently large offline dataset size ($N=200$) to ensure reliable offline estimates. The results, presented in Figure \ref{fig:exp_bias}, demonstrate that our hybrid CUCB algorithm consistently outperforms or matches the baseline performance across all tested levels of distributional bias.

\begin{figure}[h]
    \centering
\includegraphics[width=0.32\textwidth,height=0.25\textwidth]{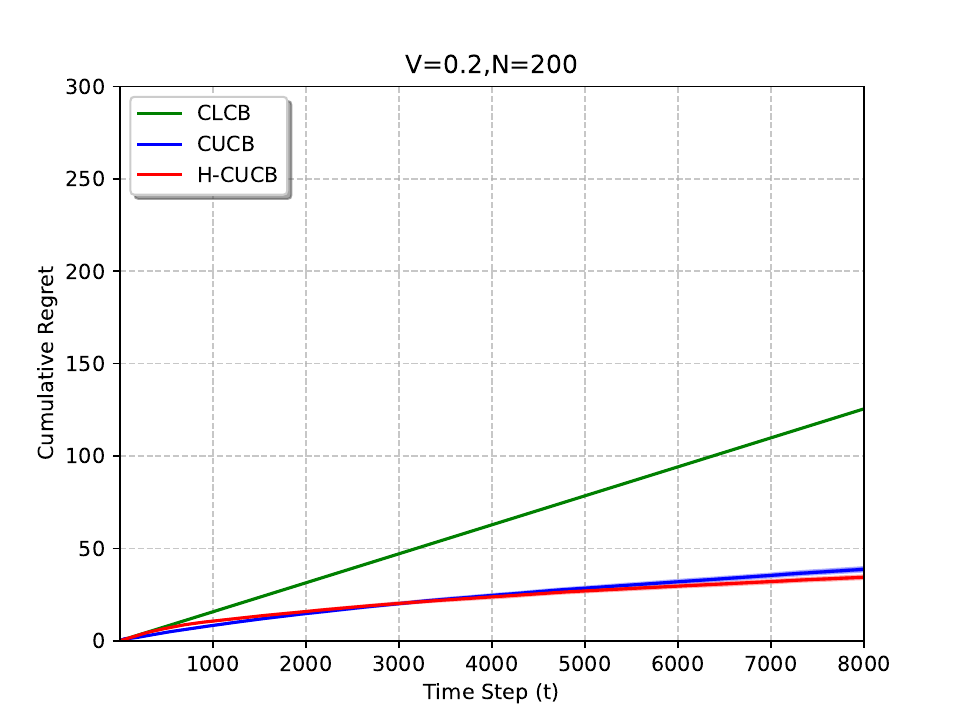}
\includegraphics[width=0.32\textwidth,height=0.25\textwidth]{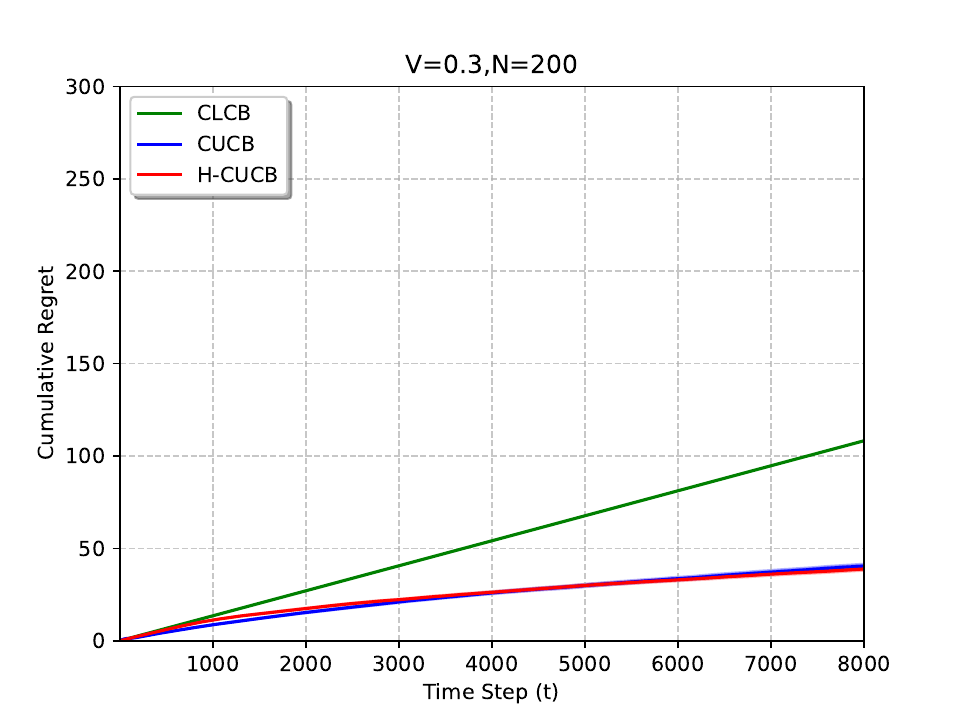}
\includegraphics[width=0.32\textwidth,height=0.25\textwidth]{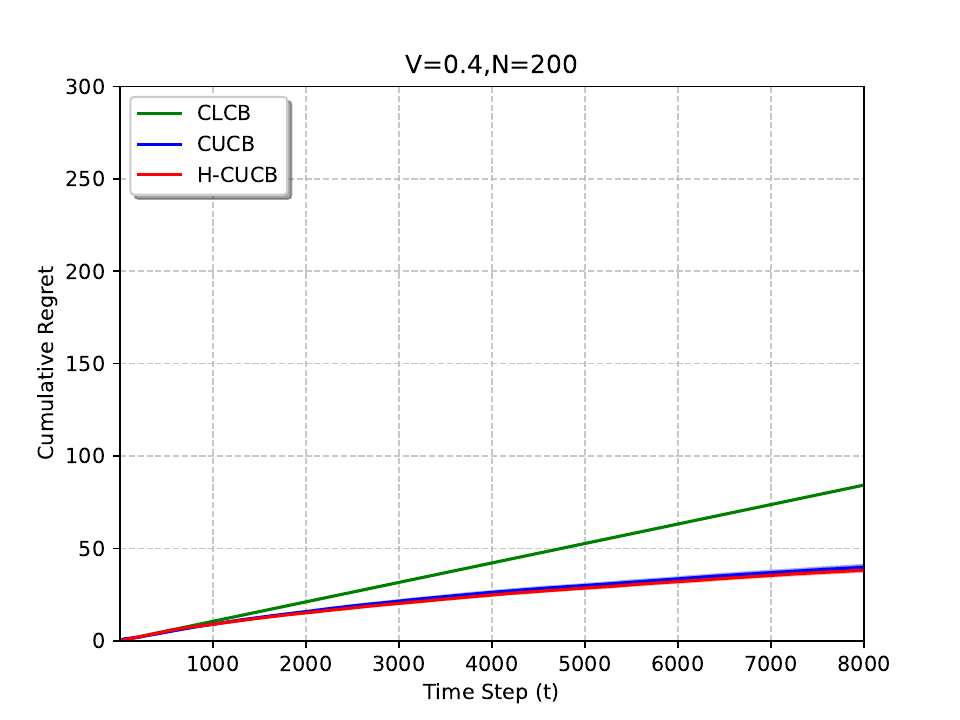}
    \caption{Performance comparison of hybrid CUCB against baselines with the biased offline data set. }
    \label{fig:exp_bias}
\end{figure}

% Finally, we validate the performance of our hybrid CUCB algorithm on a real-world offline dataset, where the offline data distribution is naturally biased. As shown in Figure 3, our algorithm again consistently outperforms or matches the baseline methods across different offline dataset sizes. Notably, hybrid CUCB achieves significantly lower regret compared to CLCB, while maintaining performance comparable to CUCB even under distributional shift. These results highlight the robustness of hybrid CUCB in practical settings, demonstrating that the algorithm can effectively leverage real-world offline data despite inherent biases and variability.

% \begin{figure}[h]
%     \centering
%    \includegraphics[width=0.32\textwidth]{figures/10.pdf}
%    \includegraphics[width=0.32\textwidth]{figures/50.pdf}
%    \includegraphics[width=0.32\textwidth]{figures/100.pdf}
%     \caption{Performance comparison of hybrid CUCB against baselines with biased real world data set. }
%     \label{fig:exp_bias}
% \end{figure}

\section{Conclusion}

We introduce H-CMAB-T, a new framework that extends classical CMAB-T by incorporating available offline data into online learning. We propose the hybrid CUCB algorithm, which selectively leverages offline observations via a minimum of two confidence bounds, controlled by a bias-aware mechanism. Theoretically, we established both gap-dependent and gap-independent regret bounds, showing that our method effectively reduces exploration through a data-dependent saving term. Empirical results further corroborate our theoretical findings, demonstrating the effectiveness of the proposed method in benchmark CMAB-T scenarios. The current CMAB-T framework does not naturally handle high-dimensional contexts or side information. Extending hybrid learning to contextual CMAB-T represents a promising direction, with potential for broader applicability in practical scenarios.

\newpage
\bibliographystyle{iclr2026_conference}
\bibliography{references} % 你的 references.bib 文件

\newpage
\appendix

\section*{Appendix}

\section{Technical Lemmas}

\begin{definition}[Event-Filtered Regret~\citep{wang2017cmabt}]
\label{def:Event-Filtered Regret}
For any series of events $\{\mathcal{E}_t\}_{t \geq 1}$ indexed by round number $t$, we define $Reg_{\mu, \alpha}^A(T, \{\mathcal{E}_t\}_{t \geq 1})$ as the regret filtered by events $\{\mathcal{E}_t\}_{t \geq 1}$, that is, regret is only counted in round $t$ if $\mathcal{E}_t$ happens in round $t$. Formally,
\[
Reg_{\mu, \alpha}^A(T, \{\mathcal{E}_t\}_{t \geq 1}) = \mathbb{E}\left[\sum_{t=1}^T \mathbb{I}(\mathcal{E}_t)(\alpha \cdot \text{opt}_\mu - r_\mu(S_t^A))\right].
\]
\end{definition}
For convenience, $A$, $\alpha$, $\mu$ and/or $T$ can be omitted when the context is clear, and we simply use $Reg_{\mu, \alpha}^A(T, \mathcal{E}_t)$ instead of $Reg_{\mu, \alpha}^A(T, \{\mathcal{E}_t\}_{t \geq 1})$.

The regret upper bound relies on considering the following events of accurate estimations by $\mathrm{UCB}_t(i)$ and $\mathrm{UCB}_t^\mathrm{S}(i)$. 
For every $t$, define:
\[
\mathcal{N}_t =  \bigcap_{i \in [m]} \left( \mathcal{N}_t(i) \cap \mathcal{N}_t^\mathrm{S}(i) \right), \quad \text{where}
\]
\[
\mathcal{N}_t(i) = \left\{ \mu^{\text{on}}_i \leq \mathrm{UCB}_t(i) \leq \mu^{\text{on}}_i + 2 \, \mathrm{rad}_t(i) \right\},
\]

\begin{equation*}
\mathcal{N}_t^\mathrm{S}(i) = \left\{
\begin{aligned}
\mu^{\text{on}}_i &\leq \mathrm{UCB}_t^\mathrm{S}(i) \leq \mu^{\text{on}}_i + \mathrm{rad}_t^\mathrm{S}(i) \\
&\quad + \left[
\sqrt{ \frac{2 \log(2t / \delta_t)}{N_i + T_{i,t-1}} }
+ \frac{ N_i \cdot \left( \mu^{\text{off}}_i - \mu^{\text{on}}_i \right) }{ N_i + T_{i,t-1} }
\right]
\end{aligned}
\right\}.
\end{equation*}
\begin{lemma}\label{lemma:good event}~\citep{cheung2024minucb}
For the event $\mathcal{N}_t$ defined above, we have $\Pr(\mathcal{N}_t) \geq 1 - 2m\delta_t.$ During the part of discussion on regret bound, we set $\delta_t=1/(2mt^2)$ for $t=1,2,...,T$.
\end{lemma}

% \begin{lemma}\label{lemma: V=1 does not use offline data}
% Under the biased good event, if $V_i = 1$, then we have $\omega_i \ge \Delta_{\min}^i / (2BK)$ for all $i$.
% \end{lemma}

% \begin{proof}
% Suppose, for contradiction, that there exists some $i$ such that
% \[
% \omega_i = V_i + \mu_i^{\text{off}} - \mu_i^{\text{on}} < \frac{\Delta_{\min}^i}{2BK}.
% \]
% Let $S^*$ be the action that satisfy $\Delta_{S^*}=\Delta_{\min}^i$. Since $V_i = 1$, this implies
% \begin{align}
% 1 + \mu_i^{\text{off}} - \mu_i^{\text{on}} 
% &< \frac{\Delta_{\min}^i}{2BK}< \frac{\Delta_{\min}^i}{BK} \\
% &\le \frac{1}{K} \sum_{j \in S^*} p_j^{D, S^*} (\mu_j^* - \mu_j^{\text{on}}) \\
% &\le \frac{1}{K} \sum_{j \in S'} p_j^{D, S'} (\mu_j' - \mu_j^{\text{on}}) \\
% &\le \mu_i' - \mu_i^{\text{on}}  \\
% &\le 1 - \mu_i^{\text{on}},
% \end{align}
% where the last step uses the fact that $\mu_i' \le 1$.

% This implies
% \[
% 1 + \mu_i^{\text{off}} - \mu_i^{\text{on}} < 1 - \mu_i^{\text{on}},
% \]
% which yields $\mu_i^{\text{off}} < 0$, contradicting the assumption that $\mu_i^{\text{off}} \ge 0$.
% \end{proof}

% \begin{lemma}\label{lemma:kappa and ell bound}
%     $\ell_{T}$ bound
% \end{lemma}

% \begin{lemma}\label{lemma:nearly optimal}
%     Why the regret bound is nearly optimal
% \end{lemma}

In the following, we provide the useful lemmas for the gap-independent regret bound. We firstly define two linear program (IP) and (LP):
\begin{align*}
\text{IP} : \max_{C_T(i),i\in[m]} &\quad \sum_{i \in [m]} \sum_{n(i) = 1}^{C_T(i)} \sqrt{\frac{1}{N_i + n(i)}} \\
\text{s.t.} &\quad \sum_{i \in [m]} C_T(i) \leq KT, \\
&\quad C_T(i) \in \mathbb{N}^+ \quad \forall i \in [m].
\end{align*}

\begin{align*}
\text{LP} : \max_{\tau,\, n} &\quad \tau \\
\text{s.t.} &\quad \tau \leq N_i + n(i) \quad \forall i\in[m] ,\\
&\quad \sum_{i \in [m]} n(i) \leq KT ,\\
&\quad \tau \geq 0,\quad n(i) \geq 0 \quad \forall i\in[m].
\end{align*}

And we suppose that $(C^*_T(i))_{i\in[m]}\in \mathbb{N}_{\ge 0}^m$ and $(\tau_*,\{n_*(i)\}_{i\in[m]})$ are the solution to (IP) and (LP) correspondingly.

\begin{lemma}\label{fact2}
For the LP defined above, we have $n_*(i) = \max \{ \tau_*  -N_i,0\}, \quad \forall i \in [m]$ .       
\end{lemma}
\begin{proof}
Since optimal solutions must be feasible, then we have $n_*(i) \ge \max \{ \tau_*  -N_i,0\}, \quad \forall i \in [m]$. We only need to prove that if $\exists$ such an arm $i'$ that $n_*(i') - \max \{ \tau_*  -N_{i'},0\}=\epsilon >0$, it will bring into a contradiction statement. In fact, we can construct another solution $(\tau',\{n'(i)\}_{i\in[m]})$ by this immediately:
\[
\left\{
\begin{aligned}
\tau' &= \tau_* + \frac{\varepsilon}{m} \\
n'(i) &= 
\begin{cases}
n_*(i) + \frac{\varepsilon}{m} & \text{if } i \ne i' \\
n_*(i) - \frac{m-1}{m} \cdot \varepsilon & \text{if } i = i'.
\end{cases} 
\end{aligned}  
\right. 
\]

Then we have $\tau'>\tau_*$, which contradicts the optimality of $\tau_*$.
\end{proof}

\begin{lemma}\label{fact1}
For the LP and IP defined above, we have $C_T^*(i) \leq \max \{\left\lceil \tau_* \right\rceil -N_i,0\}, \quad \forall i \in [m]$ .   
\end{lemma}

\begin{proof}
Suppose that there exists an arm $i'$ such that $C_T^*(i') \ge \max \{\left\lceil \tau_* \right\rceil -N_{i'},0\}+1$, then there must exist another arm $i''\ne i'$ such that $C_T^*(i'') \leq \max \{\left\lceil \tau_* \right\rceil -N_{i''},0\}-1$, or we will have:
\[
\sum_{i \in [m] } C_T^*(i) 
> \sum_{i \in [m] } \max \{\left\lceil \tau_* \right\rceil -N_i,0\}
\ge \sum_{i \in [m] } \max \{ \tau_*  -N_i,0\}
\overset{(a)}=\sum_{i \in [m] } n_*(i)=KT,
\]

which contradicts the constraint of (LP). Here, (a) is from lemma~\ref{fact2}. As a result, we can construct a feasible solution $\tilde{C}_T(i)_{i\in[m]}\in \mathbb{N}_{\ge 0}^m$ by the existence of two arms $i'$ and $i''$ that:

\[
\tilde{C}_T(i) =
\begin{cases}
C_T^*(i) - 1 & \text{if } i = i' \\
C_T^*(i) + 1 & \text{if } i = i'' \\
C_T^*(i)     & \text{if } i \in [m] \setminus \{i', i''\}.
\end{cases}
\]

By the property that $C_T^*(i') \geq 1$, $\tilde{C}_T(i') \geq 0$, and $(\tilde{C}_T(i))_{i \in [m]}$ is a feasible solution. But then we have
\begin{align*}
&\sum_{i \in [m]} \sum_{n(i) = 1}^{\tilde{C}_T(i)} \sqrt{\frac{1}{n(i) + N_i}}
-
\sum_{i \in [m]} \sum_{n(i) = 1}^{C_T^*(i)} \sqrt{\frac{1}{n(i) + N_i}} \\
&=
\sqrt{ \frac{1}{C_T^*(i'') + N_{i''} + 1} }
-
\sqrt{ \frac{1}{C_T^*(i') + N_{i'}} } \notag \\
&\geq
\sqrt{ \frac{1}{\lceil \tau_* \rceil} }
-
\sqrt{ \frac{1}{\max\{ \lceil \tau_* \rceil,\, N_{i'} \} + 1} }
> 0, 
\end{align*}

which contradicts the assumed optimality of $(N_T^*(i))_{i \in [m]}$.

\end{proof}

\section{Proof of Theorem~\ref{thm:regret-bound}}

We define the event
\[
F_t = \left\{ r_{S_t}(\bar{\mu}) \leq \alpha \cdot \mathrm{opt}_{\bar{\mu}} \right\},
\]
which captures that the oracle output based on the estimated means $\bar{\mu}$ at round $t$ achieves at least an $\alpha$-approximation of the optimal reward.

Let the filtration $\mathcal{F}_{t-1}$ represent all the history observed up to and including the decision $S_t$, formally:
\[
\mathcal{F}_{t-1} = \left( S_1, \tau_1, \{ X_{1,i} : i \in \tau_1 \}, \ldots, S_{t-1}, \tau_{t-1}, \{ X_{t-1,i} : i \in \tau_{t-1} \}, S_t \right).
\]
Here, $\tau_s$ denotes the triggered set at round $s$, and $X_{s,i}$ is the observed reward for arm $i$ in round $s$ if triggered. We emphasize that the filtration $\mathcal{F}_{t-1}$ already implicitly incorporates the information from the offline data.
In particular, the observations of arm $i$ offline affect the initialization of arm statistics such as $\mathrm{rad}_t^{\mathrm{S}}$and 
$\mathrm{UCB}_t^{\mathrm{S}}$,
which in turn influence the selection of $S_t$ at each round $t$.
Therefore, the subsequent triggered sets $\tau_t$ and observed rewards $\{ X_{t,i} : i \in \tau_t \}$ are also conditioned on the offline data through the choice of $S_t$.

The conditional expectation at round $t$ is defined as
\[
\mathbb{E}_t[\cdot] = \mathbb{E}\left[ \cdot \mid \mathcal{F}_{t-1} \right],
\]
which aligns with the algorithm's access to the complete history $\mathcal{F}_{t-1}$ when making decisions at round $t$. Moreover, quantities such as $S_t$ and $\bar \mu_{i,t}$ are $\mathcal{F}_{t-1}$-measurable.

\begin{proof}
Since the $\mu_i^\text{on}$ is the actual mean we focus to learn about for every arm $i$, we set $\mu_i=\mu_i^\text{on}$ for every arm $i$ for simplicity. To unify the proofs for the proof of Theorem~\ref{thm:regret-bound} and the proof of bound $\psi$ in Theorem~\ref{thm:gap-indep}, we introduce a positive parameter $M_i$ for every arm $i$, which is introduced in \cite{wang2017cmabt}. We also further inherit the definition $M_S:=\max_{i \in \tilde{S}} M_i$ for each action $S$ and $M_S=0$ if $\tilde{S}= \varnothing$ from \cite{wang2017cmabt}. 

We first show that if $\{S_t \ge M_{S_t}\}$, $\mathcal{N}_t$ and $\neg F_t$, and given filtration $\mathcal{F}_{t-1}$, we have:

% \begin{align}
% \Delta_{S_t} = \mathbb{E}_t[\Delta_{S_t}]
% &  \overset{(a)}{\leq} \mathbb{E}_t \left[2B\sum_{i \in \tilde{S_t}} \left[ p_i^{D,S_t} (\bar \mu_{i,t}-\mu_i)-\frac{M_i}{2BK} \right]\right] \label{similar part 1} \\
% & \overset{(b)}{\leq} \mathbb{E}_t \left[2B\sum_{i \in \tilde{S_t}}p_i^{D,S_t}\left[(\bar \mu_{i,t}-\mu_i)-\frac{M_i}{2BK}\right]\right] \\
% &\overset{(c)}{=} \mathbb{E}_t \left[ 2B \sum_{i \in \tilde{S_t}} \mathbb{I}\{i \in \tau_t\} \left[(\bar \mu_{i,t}-\mu_i) -\frac{M_i}{2B K} \right] \right] \\
% &\overset{(d)}{=} \mathbb{E}_t \left[ 2B \sum_{i \in \tau_t}  \left[(\bar \mu_{i,t}-\mu_i) -\frac{M_i}{2B K} \right] \right] , \label{similar part 2} 
% \end{align}

\begin{align}
\Delta_{S_t} = \mathbb{E}_t[\Delta_{S_t}]
&  {\leq} \mathbb{E}_t \left[2B\sum_{i \in \tilde{S_t}} \left[ p_i^{D,S_t} (\bar \mu_{i,t}-\mu_i)-\frac{M_i}{2BK} \right]\right] \label{eq:deta_{S_t}1} \\
& {\leq} \mathbb{E}_t \left[2B\sum_{i \in \tilde{S_t}}p_i^{D,S_t}\left[(\bar \mu_{i,t}-\mu_i)-\frac{M_i}{2BK}\right]\right] \label{eq:deta_{S_t}2}\\
&{=} \mathbb{E}_t \left[ 2B \sum_{i \in \tilde{S_t}} \mathbb{I}\{i \in \tau_t\} \left[(\bar \mu_{i,t}-\mu_i) -\frac{M_i}{2B K} \right] \right] \label{eq:deta_{S_t}3}\\
&{=} \mathbb{E}_t \left[ 2B \sum_{i \in \tau_t}  \left[(\bar \mu_{i,t}-\mu_i) -\frac{M_i}{2B K} \right] \right] , \label{eq:deta_{S_t}4} 
\end{align}

where (\ref{eq:deta_{S_t}1}) comes from exactly the equation (11) of Appendix B.3 in~\cite{wang2017cmabt}, (\ref{eq:deta_{S_t}2})  comes from the fact that $p_i^{D,S_t}\leq 1$ for every arm $i$, (\ref{eq:deta_{S_t}3})  follows from the fact that since the algorithm choose $S_t$ using the information of offline data, then $S_t$ and $\bar \mu_{i,t}$ are $\mathcal{F}_{t-1}$ measurable and the only randomness is the triggering set $\tau_t$ at round $t$, which satisfies the conditions of TPE trick in ~\cite{liu2023contextualCMAB}, so we can also use TPE trick~\citep{liu2023contextualCMAB} to replace $p_i^{D,S_t}= \mathbb{E}_t[\mathbb{I}\{i\in\tau_t\}]$, and (\ref{eq:deta_{S_t}4})  is the change of notion $\tau_t$.

Then we use $\kappa$ to describe the concentration for $(\bar \mu_{i,t}-\mu_i)-{M_i}/{(2BK)}$, the intuition is that we use different UCBs depending on how informative the offline data is.

\paragraph{Case 1} When $\omega_i < M_i/(2BK)$, let

\[
\hspace{0.6cm} \kappa _{T,N_i,\omega_i}(M_i,s) =
\begin{cases} 
    4B, & s = N_i = 0 \\
   4B\sqrt{\frac{2\log(4mT^3)}{N_i+s}}, & 0 \leq s \leq \ell_{T,N_i,\omega_i}(M_i)\\
    0, & s > \ell_{T,N_i,\omega_i}(M_i)~ or ~ \ell_{T,N_i,\omega_i}(M_i) \leq 0,
\end{cases}
\]

\hspace{0.2cm} where \[
\ell_{T,N_i,\omega_i}(M_i) = \frac{64 B^2 K^2 \log (4mT^3)}{M_i^2}-N_i\cdot \max\{1-\frac{2BK\omega_i}{M_i},0\}^2.
\]

We first prove when $\omega_i < M_i/(2BK)$, we have $(\bar \mu_{i,t}-\mu_i)-{M_i}/{(2BK)}\leq  \kappa _{T,N_i,\omega_i}(M_i,t)$:

Since $\omega_i < \frac{M_i}{2BK}$, $ \max \{ 1 - \frac{2BK\omega_i}{M_i}, 0 \}^2 = \left( 1 - \frac{2BK\omega_i}{M_i} \right)^2 > 0$,

\begin{enumerate}
    \item if $N_i \cdot \left( 1 - \frac{2BK \omega_i}{M_i} \right)^2 \geq \frac{32 B^2 K^2 \log (4mT^3)}{M_i^2}$, then we have:

    \begin{equation}\label{case1.1}
    2 \sqrt{\frac{2 \log (4mT^3)}{N_i + T_{i,t-1}}} \leq 2 \sqrt{\frac{2 \log (4mT^3)}{N_i}} \leq \left( 1 - \frac{2BK \omega_i}{M_i} \right) \cdot \frac{M_i}{2BK}.
    \end{equation}

    \begin{align*}
    \Rightarrow \bar{\mu}_{i,t} \overset{(a)}{\leq} \mathrm{UCB}^{\mathrm{S}}_t(i) 
    &\overset{(b)}{\leq} \mu_i + 2 \sqrt{\frac{2 \log (4mT^3)}{N_i + T_{i,t-1}}} + \frac{N_i}{N_i + T_{i,t-1}} \cdot \omega_i  \\
    &\overset{(c)}{\leq} \mu_i + \left( 1 - \frac{2BK \omega_i}{M_i} \right) \cdot \frac{M_i}{2BK} + \omega_i \\
    &=  \mu_i +\frac{M_i}{2BK} ,
    \end{align*}

    \begin{align*}
    \Rightarrow \bar{\mu}_{i,t} - \mu_i +\frac{M_i}{2BK} {\leq}0\leq \kappa _{T,N_i,\omega_i}(M_i,t) 
    \end{align*}

    where (a) comes from the definition of $\bar{\mu}_{i,t}$ in Algorithm~\ref{alg:hybrid cucb}, (b) follows from the lemma~\ref{lemma:good event} and the definition of $\omega_i$, and (c) follows from~(\ref{case1.1}) and $N_i/(N_i+T_{i,t-1})\leq1$.

    \item when $N_i \cdot \left( 1 - \frac{2BK \omega_i}{M_i} \right)^2 < \frac{32 B^2 K^2 \log (4mT^3)}{M_i^2}$, then we have:

    \begin{equation*}
         \ell_{T,N_i,\omega_i}(M_i) = \frac{64 B^2 K^2 \log (4mT^3)}{M_i^2} - N_i \max\left( 1 - \frac{2BK\omega_i}{M_i}, 0 \right)^2 > \frac{32 B^2 K^2 \log (4mT^3)}{M_i^2}.
    \end{equation*}    

    (1) when $T_{i,t-1} >\ell_{T,N_i,\omega_i}(M_i)> \frac{32 B^2 K^2 \log (4mT^3)}{M_i^2}$, we have :

    \begin{align}
    \bar{\mu}_{i,t} \overset{(a)}{\leq} \mathrm{UCB}_t(i) 
    \nonumber &\overset{(b)}{\leq}  \mu_i+ 2\sqrt{ \frac{2\log (4mt^3)}{T_{i,t-1}} } \\
    &\overset{(c)}{<}  \mu_i+\frac{M_i}{2BK}, \label{case1.2.(1)}
    \end{align}

    \begin{align*}
    \Rightarrow \bar{\mu}_{i,t} - \mu_i +\frac{M_i}{2BK} {\leq}0\leq \kappa _{T,N_i,\omega_i}(M_i,t) 
    \end{align*}

    where (a) comes from the definition of $\bar{\mu}_{i,t}$ in Algorithm~\ref{alg:hybrid cucb}, (b) follows from lemma~\ref{lemma:good event}, and (c) follows from the condition that $T_{i,t-1} > \frac{32 B^2 K^2 \log (4mT^3)}{M_i^2}$.

    (2) when $0\leq T_{i,t-1} \leq \ell_{T,N_i,\omega_i}(M_i)$, then we have:

    \begin{align*}
    \bar{\mu}_{i,t} - \mu_i - \frac{M_i}{2BK}
    &\overset{(a)}{\leq} \mathrm{UCB}^\mathrm{S}_t(i) - \mu_i - \frac{M_i}{2BK} \\
    &\overset{(b)}{\leq} 2\sqrt{\frac{2\log(4mT^3)}{N_i + T_{i,t-1}}} + \frac{N_i}{N_i + T_{i,t-1}} \omega_i - \frac{M_i}{2BK} \\
    &\overset{(c)}{\leq} 2\sqrt{\frac{2\log(4mT^3)}{N_i + T_{i,t-1}}} + \frac{N_i}{N_i + T_{i,t-1}} \cdot\frac{M_i}{2BK}  - \frac{M_i}{2BK} \\
    &= 2\sqrt{\frac{2\log(4mT^3)}{N_i + T_{i,t-1}}} + (\frac{N_i}{N_i + T_{i,t-1}} -1)\cdot\frac{M_i}{2BK}  \\
    &\overset{(d)}{\leq} 2\sqrt{\frac{2\log(4mT^3)}{N_i + T_{i,t-1}}} , 
    \end{align*}

    where (a) comes from the definition of $\bar{\mu}_{i,t}$ in Algorithm~\ref{alg:hybrid cucb}, (b) follows from lemma~\ref{lemma:good event}, (c) follows from the condition that $\omega_i < M_i/(2BK)$ and (d) follows from $N_i/(N_i+T_{i,t-1})-1\leq 0$.

\end{enumerate}

\paragraph{Case 2} When $\omega_i> M_i/(2BK)$, we firstly define:

\hspace{0.2cm}\[
\kappa _T(M_i, s) =
\begin{cases} 
    4B, & s = 0 \\
    4B{\sqrt{\frac{2\log (4mT^3)}{s}}}, &1 \leq s \leq \ell_T(M_i) \\
    0, & s > \ell_T(M_i) ,
\end{cases}
\]

\hspace{0.2cm} where \[
\ell_T(M_i) = \frac{32 B^2 K^2 \log (4mT^3)}{M_i^2}.
\]

Since $\omega_i \ge \frac{M_i}{2BK}$, $ \max \{ 1 - \frac{2BK\omega_i}{M_i}, 0 \}^2 =  0$. Follow the similar analysis in \textbf{case 1}, we have:

\begin{align*}
\bar{\mu}_{i,t} - \mu_i - \frac{M_i}{2BK}
&\leq\mathrm{UCB}_t(i) - \mu_i - \frac{M_i}{2BK} \\
&\leq 2\sqrt{\frac{2\log(4mT^3)}{T_{i,t-1}}}  - \frac{M_i}{2BK} \\
&\overset{(a)}{\leq} \kappa _T(M_i, s),
\end{align*}

where (a) follows from: if  $T_{i,t-1} >\ell_{T}(M_i)=\frac{32 B^2 K^2 \log (4mT^3)}{M_i^2}$, then from~(\ref{case1.2.(1)}) we have $\bar{\mu}_{i,t} - \mu_i - {M_i}/{(2BK)}\leq0$, and if $T_{i,t-1} \leq \ell_{T}(M_i)$ we have $2\sqrt{\frac{2\log(4mT^3)}{T_{i,t-1}}}  - \frac{M_i}{2BK} \leq 2\sqrt{\frac{2\log(4mT^3)}{T_{i,t-1}}}$.

Notice that if we set $N_i=0$, then the definition of $\kappa$ in \textbf{case 1} can cover the definition of $\kappa$ in \textbf{case 2}. As a result, we use $ \kappa _{T,N_i,\omega_i}(M_i,s)$ for the following statement for simplicity. From the \textbf{case 1} and \textbf{case 2}, then we can derive the regret into two parts:
\begin{align*}
\operatorname{Reg}(\{S_t\ge M_{S_t}\},\mathcal{N}_t, \neg F_t) 
&= \mathbb{E}\left[\sum_{t=1}^T \Delta_{S_t}\right] \\
&\overset{(a)}{\leq} \mathbb{E} 
\left[\sum_{t=1}^T \mathbb{E}_t \left[  
\sum_{i \in \tau_t} \kappa _{T,N_i,\omega_i}(M_i,T_{i,t-1}) \right] \right] \\
&\overset{(b)}{=} \mathbb{E} 
\left[\sum_{t=1}^T 
\sum_{i \in \tau_t} \kappa _{T,N_i,\omega_i}(M_i,T_{i,t-1})  \right]  \\
&\overset{(c)}{=} \mathbb{E}\left[\sum_{i \in [m]} \sum_{s=0}^{T_{T-1,i}} \kappa _{T,N_i,\omega_i}(M_i,s) \right] \\      
&\leq \sum_{i \in [m]} \sum_{s=0}^{T_{T-1,i}} \kappa _{T,N_i,\omega_i}(M_i,s)  \\  
&= \sum_{t=1}^T \left[ 
\sum_{i \in \tilde{S_t},\omega_i \leq M_i/(2BK)} \kappa _{T,N_i,\omega_i}(M_i,t) +  
\sum_{i \in \tilde{S_t},\omega_i> M_i/(2BK)} \kappa _T(M_i, t) \right] \\
&:= \underline{A} +\underline{B},
\end{align*}

where (a) follows from the discussion on \textbf{case 1} and \textbf{case 2}, (b) follows from the tower rule, (c) follows from that $T_{t-1,i}$ is increased by 1 if and only if the arm $i$ is triggered at round $t$.

% 这个是不统一kappa写法的情况，式子会很长
% \begin{align}
% \operatorname{Reg}(\mathcal{N}_t^s, \neg F_t) 
% &= \mathbb{E}\left[\sum_{t=1}^T \Delta_{S_t}\right] \\
% &\overset{(a)}{\leq} \mathbb{E} 
% \left[\sum_{t=1}^T \mathbb{E}_t \left[ \sum_{i \in \tau_t,\omega_i> M_i/(2BK)} \kappa _T(M_i, T_{i,t-1}) +  
% \sum_{i \in \tau_t,\omega_i \leq M_i/(2BK)} \kappa _{T,N_i,\omega_i}(M_i,T_{i,t-1}) \right] \right] \\
% &\overset{(b)}{=} \mathbb{E} 
% \left[\sum_{t=1}^T 
% \left[\sum_{i \in \tau_t,\omega_i> M_i/(2BK)} \kappa _T(M_i, T_{i,t-1}) +  
% \sum_{i \in \tau_t,\omega_i \leq M_i/(2BK)} \kappa _{T,N_i,\omega_i}(M_i,T_{i,t-1})  \right] \right] \\
% &\overset{(c)}{=} \mathbb{E}\left[\sum_{i \in [m]} \sum_{s=0}^{T_{T-1,i}} \kappa_{T,N_i}(M_i,t)\right] \\         
% &\leq \sum_{t=1}^T \left[ \sum_{i \in \tilde{S_t},\omega_i> M_i/(2BK)} \kappa _T(M_i, t) +  
% \sum_{i \in \tilde{S_t},\omega_i \leq M_i/(2BK)} \kappa _{T,N_i,\omega_i}(M_i,t) \right] \\
% &:= \underline{A} +\underline{B}
% \end{align}

We then compute \underline{A}+\underline{B}:

For part $\underline{A}$:
\[
\underline{A} = \sum_{\omega_i \leq M_i / 2BK} \sum_{s=0}^{\ell_{T,N_i,\omega_i}(M_i)}  \kappa _{T,N_i,\omega_i}(M_i,s).
\]

For simplicity we set $N_i'=N_i\cdot \max\{1-\frac{2BK\omega_i}{M_i},0\}^2$. Since when $N_i' \ge \ell_{T,N_i,\omega_i}(M_i) $ $\kappa _{T,N_i,\omega_i}(M_i,s)=0$, we consider the case that $N_i' < \ell_{T,N_i,\omega_i}(M_i) $ :

\begin{align*}
\sum_{s=0}^{\ell_{T,N_i,\omega_i}(M_i)}  \kappa _{T,N_i,\omega_i}(M_i,s) 
&= 4B\sqrt{2\log (4mT^3)}\sum_{s=1}^{\ell_{T,N_i,\omega_i}(M_i)}  \frac{1}{\sqrt{N_i+s}} ds  +4B \\
&\overset{(a)}{\leq}  4B\sqrt{2\log (4mT^3)}\int_{0}^{\ell_{T,N_i,\omega_i}(M_i)}  \frac{1}{\sqrt{N_i+s}} ds  +4B \\
&\overset{(b)}{\leq}   4B\sqrt{2\log (4mT^3)}\int_{0}^{\ell_{T,N_i,\omega_i}(M_i)}  \frac{1}{\sqrt{N_i'+s}} ds +4B\\ 
&=  \frac{64 \sqrt{2} B^2 K \log(4mT^3)}{M_i} 
- 8B \sqrt{2N_i' \log(4mT^3)} +4B ,
\end{align*}

where (a) is by the sum \& integral inequality $\int_{L-1}^U f(x) dx \geq \sum_{i=L}^U f(i) \geq \int_L^{U+1} f(x) dx$ for non-increasing function $f$, (b) follows from $N_i'\leq N_i$ and the monotonicity of integrals.

For part $\underline{B}$, similar as part \underline{A}, we have:
\[
\underline{B} = \sum_{\omega_i > M_i / 2BK} \sum_{s=0} ^{\ell_T(M_i)}\kappa _T(M_i, s) \leq \sum_{\omega_i >  M_i / 2BK} \left( \frac{64 \sqrt{2} B^2 K \log (4mT^3)}{M_i} + 4B \right).
\]    

Sum up \underline{A} + \underline{B}, plus the case that $N_i'$ may be $\ge \ell_{T,N_i,\omega_i}(M_i):=\frac{64B^2K^2\log{4mT^3}}{M_i^2} $, we have

\begin{equation}\label{gap-dependent good event}
\operatorname{Reg}(\{S_t\ge M_{S_t}\},\mathcal{N}_t, \neg F_t)  \leq \sum_{i \in [m]} \max\left\{ \frac{64\sqrt{2} B^2 K \log(4mT^3)}{M_i} - 8B \sqrt{2N_i'\log(4mT^3)},0 \right\}+4Bm.
\end{equation}

For the gap-dependent bound, take $M_i=\Delta_{\min}^i$, then $\operatorname{Reg}(S_t< M_{S_t} )=0$. And following \cite{wang2017cmabt} to handle small probability events $\neg \mathcal{N}_t$ and $F_t$ we have
\begin{equation}\label{gap-dependent final}
    \operatorname{Reg}(T) \leq \sum_{i \in [m]} \max\left\{ \frac{64\sqrt{2} B^2 K \log (4mT^3)}{\Delta_{\min}^i} - 8B \sqrt{2N_i'\log (4mT^3)},0 \right\} +4Bm+\frac{\pi^2}{6}\Delta_{\max},
\end{equation}
where 
\[
N_i' = N_i \cdot \max\left\{ 1 - \frac{2BK \omega_i}{\Delta_{\min}^i}, 0 \right\}^2.
\]
\end{proof}

\section{Proof of Theorem ~\ref{thm:gap-indep}}

To prove Theorem~\ref{thm:gap-indep}, we present two candidate regret bounds, each derived via a distinct analysis technique. We denote these bounds as $\psi$ and $\gamma$, and show that the regret is upper bounded by the minimum of the two.

\subsection{Proof of Bound $\psi$}

\begin{proof}
We further discuss~(\ref{gap-dependent good event}) and~(\ref{gap-dependent final}). For the gap-independent bound, take $M_i=M= \sqrt{64 \sqrt{2} m B^2 K \log (4mT^3)/{T}}$, then $\operatorname{Reg}(S_t< M_{S_t} )\leq TM$. (Naturally the $N_i'$ would change correspondingly.) Then we have
\begin{align*}
\operatorname{Reg}(T) 
&\leq \sum_{i \in [m]} \max\left\{ \frac{64\sqrt{2} B^2 K \log (4mT^3)}{M_i} - 8B \sqrt{2N_i''\log (4mT^3)},0 \right\} + \operatorname{Reg}(S_t< M_{S_t}) \\
&\leq \sum_{i \in [m]} \max\left\{ \frac{64 \sqrt{2} B^2 K \log (4mT^3)}{M_i} - 8B \sqrt{2N_i''\log (4mT^3)},0 \right\}+ TM \\
&\leq 8 \sqrt{2}B \sqrt{\log (4mT^3)} \left( \sum_{i \in [m]} \max \left\{ \sqrt{\frac{KT}{m}} - \sqrt{N_i''}, 0 \right\} + \sqrt{mKT} \right),
\end{align*}

where 
\[
N_i'' = N_i \cdot \max\left\{ 1 - \frac{\omega_i}{4 \sqrt{2}}\sqrt{\frac{KT}{m \log(4mT^3)}}, 0 \right\}^2.
\]    
\end{proof}

\subsection{Proof of Bound $\gamma$}
\paragraph{Intuition.} 
The key idea behind the $\gamma$ bound lies in adopting a different perspective for establishing early stopping conditions. In the gap-dependent analysis, the number of times each arm $i$ needs to be triggered is directly related to its gap. However, when such gap information is unavailable, we must seek alternative ways to characterize how offline data effectively reduces the required online exploration for each arm.

To this end, we observe that the regret incurred by an arm depends on both the amount of offline data $N_i$ and the number of times it is triggered online $T_i$. Motivated by this and \cite{cheung2024minucb}, a formulation based on a linear program is obtained, which captures how much exploration can be saved through leveraging informative offline data, even without explicit gap knowledge.

\begin{proof}
Under the events $\mathcal{N}_t$ and $\neg F_t$, and given filtration $\mathcal{F}_{t-1}$, follow the similar analysis from~(\ref{eq:deta_{S_t}1}) to~(\ref{eq:deta_{S_t}4}) we have:

\begin{align}
\Delta_{S_t} = \mathbb{E}_t[\Delta_{S_t}]
\nonumber&{\leq} \mathbb{E}_t \left[B\sum_{i \in \tilde{S_t}}  p_i^{D,S_t} (\bar \mu_{i,t}-\mu_i) \right] \\ 
\nonumber&{=} \mathbb{E}_t \left[ B \sum_{i \in \tilde{S_t}} \mathbb{I}\{i \in \tau_t\} \left[(\bar \mu_{i,t}-\mu_i) \right] \right] \\
&{=} \mathbb{E}_t \left[ B \sum_{i \in \tau_t}  (\bar \mu_{i,t}-\mu_i)\right] .\label{short sybmbol}
\end{align}

Focusing on the regret analysis on $\mathrm{UCB}^\mathrm{S}$ then we have:

\begin{align*}
\operatorname{Reg}(\{S_t\ge M_{S_t}\},\mathcal{N}_t, \neg F_t) 
&= \mathbb{E}\left[\sum_{t=1}^T \Delta_{S_t}\right] \\
&\overset{(a)}{\leq} \mathbb{E} 
\left[\sum_{t=1}^T \mathbb{E}_t  
\sum_{i \in \tau_t}  B(\bar \mu_{i,t}-\mu_i) \right] \\
&\overset{(b)}{\leq} \mathbb{E} \left[
B\sum_{t=1}^T 
\sum_{i \in \tau_t} (\mathrm{UCB}^\mathrm{S}_t(i)-\mu_i) \right] \\
&{\leq} B\sum_{t=1}^T 
\sum_{i \in \tau_t} (\mathrm{UCB}^\mathrm{S}_t(i)-\mu_i) ,
\end{align*}

where (a) follows from~(\ref{short sybmbol}), and (b) follows from the tower rule and $\bar \mu_{i,t}\leq\mathrm{UCB}^\mathrm{S}_t(i)$. Follow from the lemma~\ref{lemma:good event}, we have:
\begin{align}
\nonumber &B\sum_{t=1}^T \sum_{i \in \tau_t }(\mathrm{UCB}^\mathrm{S}_t(i)-\mu_i) \\
\nonumber&\overset{(a)}{\leq}B\sum_{t=1}^T \sum_{i \in \tau_t }
(2\sqrt{\frac{2\log (4mt^3)}{N_i+T_{i,t-1}}}+\frac{N_i}{N_i+T_{i,t-1}} \omega_i) \\
\nonumber&\overset{(b)}{\leq} 2B \sum_{t=1}^T \sum_{i \in \tau_t }
\sqrt{\frac{2\log (4mt^3)}{N_i+T_{i,t-1}} }+BKT\omega_{\max} \\ 
\nonumber&\overset{(c)}{=} 2B\sum_{i \in [m]} \sum_{n(i)=0}^{C_T(i)}
\sqrt{\frac{2\log (4mt^3)}{N_i+n(i)}}+BKT\omega_{\max}  \\
&\overset{(d)}{\leq} 2B\sum_{i \in [m]} \sum_{n(i)=1}^{C_T(i)}
\sqrt{\frac{2\log (4mt^3)}{N_i+n(i)}}+BKT\omega_{\max} +2Bm, \label{(1)}
\end{align}

Where (a) is from lemma~\ref{lemma:good event}, (b) follows from $N_i+T_{i,t-1} \ge N_i$ and the definition of $K$, $C_T(i)$ in (c) is an undetermined coefficient discuss next, and (d) considers the case that $n(i)+N_i$ may equal to zero, which in that case we treat the log-term as one as algorithm~\ref{alg:hybrid cucb} designed in line~\ref{alg:line:radius:hybird}.

And we use linear program to consider the $C_T(i)$:
\begin{align}
\sum_{i \in [m]} \sum_{n(i)=1}^{C_T(i)}
\frac{1}{\sqrt{N_i+n(i)}}
\nonumber &\overset{(a)}{\leq} \sum_{i \in [m]} \sum_{n(i)=1}^{C_T^*(i)}
\frac{1}{\sqrt{N_i+n(i)}}  \\
\nonumber &\overset{(b)}{\leq} \sum_{i \in [m]} \sum_{n(i)=1}^{\max \{\left\lceil \tau_* \right\rceil -N_i,0\}} \frac{1}{\sqrt{N_i+n(i)}}  \\
\nonumber &\leq \sum_{i \in [m]} \frac{ \max \{\left\lceil \tau_* \right\rceil -N_i,0\}}{\left\lceil \tau_* \right\rceil} \sum_{t=1}^{\left\lceil \tau_* \right\rceil}\frac{1}{\sqrt{t}} \\
\nonumber &\leq  \sum_{i \in [m]} { \max \{\left\lceil \tau_* \right\rceil -N_i,0\}}\cdot\frac{4}{\sqrt{\tau_*}} \\
\nonumber &\leq  \sum_{i \in [m]} {( \max \{ \tau_* -N_i,0\}}+1)\cdot\frac{4}{\sqrt{\tau_*}} \\
\nonumber &\overset{(c)}{\leq} \sum_{i \in [m]} {( n_*(i)}+1)\cdot\frac{4}{\sqrt{\tau_*}} \\
&\leq \frac{8KT}{\sqrt{\tau_*}} , \label{(2)}
\end{align}

where (a) comes from the definition of (LP), (b) from lemma~\ref{fact1}, (c) follows from the feasibility of $(\tau_*,\{n_*(i)\}_{i\in[m]})$ to (LP).

Combine (\ref{(1)}) and (\ref{(2)}), and following \cite{wang2017cmabt} to handle small probability events $\neg \mathcal{N}_t$ and $F_t$ then we the final regret bound:

\begin{align*}
\operatorname{Reg}(T) \leq 16BT\sqrt{\frac{2 \log (4mT^3)}{\tau_*}}+BKT\omega_{\max}+4Bm+\frac{\pi^2}{6}\Delta_{\max} .
\end{align*}

\end{proof}
% 在这里写 LP 上界的构造和证明方法（如 dual argument 或 packing LP）

\section{Experimental Details and Real-World Validation}
 
% \subsection{Experimental Setup}

%  The experiments were conducted on a high-performance workstation equipped with an AMD Ryzen Threadripper PRO 7985WX CPU, featuring 64 physical cores and 128 threads with a maximum clock speed of up to 5.367\,GHz. The system is configured with 124\,GB of RAM, providing ample memory for parallel processing and large-scale simulation tasks. For storage, the workstation utilizes a high-speed 926\,GB NVMe SSD, with 761\,GB available during the experiments, ensuring fast read/write access for intermediate data, logs, and visualization outputs. This configuration meets and exceeds the requirements for compute-intensive tasks, including those involving multi-process simulations, high-throughput memory access, and frequent disk I/O. The use of Python's \texttt{multiprocessing} library takes advantage of the multi-core architecture to execute multiple independent trials in parallel, significantly improving experiment throughput and reducing runtime. Overall, the workstation's specifications ensure smooth and efficient execution of all simulation and evaluation pipelines.
% \subsection{Experiment Design}

We compare our proposed hybrid CUCB with existing CUCB for the pure online setting \citep{wang2017cmabt} and CLCB for the pure offline setting \citep{liu2025offlineCMAB}. 
% There are many applications of CMAB-T, such as Learning to Rank, LLM Cache and Social Influence Maximization.
We mainly focus on the task of online learning to rank for the considered CMAB-T problem, where the agent selects $k$ from $m$ base arms. The outcome distribution of each base arm is Bernoulli. We set $m=10$ and $k=5$. All results are averaged over $20$ runs, and the error bar is defined as the standard deviation divided by $\sqrt{20}$. 
% For real world applications, we use MovieLens as data set. 
The triggering process and reward function are introduced as below following existing literature \citep{chen2016combinatorial,liu2025offlineCMAB}: 
\begin{itemize}
    \item Triggering process: The super arm $S_t$ is a permutation over $k$ arms. The environment would check the Bernoulli outcome from the first to the last one. If the first arm has outcome $1$, then the triggering stops. Otherwise, the environment would check the second arm. Similar process continues until one arm has the outcome $1$. All arms ranked before this arm are observed with outcome $0$ and this arm is observed with outcome $1$. The following arms have no observations. 
    % \item 
    % For the super arm each t, we choose them to do the trigger process, which means for each $i\in S_t,$ we do Bernoulli
    % sampling according to $\mu^{\text{on}}_i$. First, put such arm into the observed arm set then do sampling. If the sampling result is 0, set $X_i^{(t)}=0$ and choose another arm from $S_t$ until the result of Bernoulli sampling is 1. Set the arm $X_i^{(t)}=1$ then stop the triggering process and update the $T_i$ and $\hat{\mu}^{\text{on}}_i$ for each $i\in S_t$.
    \item Reward function: The reward function is defined as:
\begin{equation*}
    r(S_t,\mu)=1-\prod_{i\in S_t}(1-\mu_i). 
\end{equation*}
% \item It is also important to note that the algorithm comparison relies somewhat on chance, as it requires the CLCB algorithm to occasionally "make a mistake." If CLCB has already identified the best 5 arms, there is no room for improvement, and the comparison becomes less meaningful. Therefore, when conducting the experiment, one must patiently wait for CLCB to select at least one suboptimal arm—i.e., an arm outside the top 5—so that the performance gap is greater than zero and the comparison becomes informative.
\end{itemize}

% The objective is to select $k$ from $m$ base arms every turn to explore the distribution of arms online, and minimum the cumulative regret. The experiments of the paper focuses on Learning to Rank for simplicity.

% Before running each experiment, we reset all stored variables to ensure a clean environment. We set $m = 10$, $k = 5$, and use 20 CPU cores for parallelization.

We evaluate the performances of algorithms in both unbiased ($V=0$) and biased ($V\ne 0$) environments. 
For the unbiased case, we generate $\mu_i^{\text{on}}$ uniformly in the interval $(0, 0.5)$ and set $\mu_i^{\text{off}} = \mu_i^{\text{on}}$. For the biased setting, we test different values of discrepancy $V\in \{0.2,0.3,0.4\}$. To ensure both $\mu_i^{\text{on}}$ and $\mu_i^{\text{off}}$ fall into interval $[0,1]$ when evaluating different values of $V_i$, we generate $\mu_i^{\text{on}}$ uniformly in the interval $(0.4, 0.5)$ and uniformly choose $V_i=V$ or $V_i=-V$. We set $\mu_i^{\text{off}}=\mu_i^{\text{on}}+V_i$. 

% For the offline data set. We first determine the value $N$ which representing the offline observations for each base arm $i$. And then independently sample $N$ observations with mean $\mu_i^{\text{off}}$. 

Finally, we validate the performance of our hybrid CUCB algorithm on a real-world dataset. Specifically, we use the MovieLens dataset, where we randomly 
select 10 movies as the arms and split the data into two disjoint parts to represent 
online and offline feedback. The bias level $V$ is computed as the mean difference 
between the two parts. As shown in Figure~\ref{fig:exp_bias}, our algorithm 
consistently outperforms or matches the baselines across different offline dataset 
sizes. Notably, hybrid CUCB achieves significantly lower regret compared to CLCB, 
while maintaining performance comparable to CUCB even under distributional shift. 
These results highlight the robustness of hybrid CUCB in practical settings, 
demonstrating that the algorithm can effectively leverage real-world offline data 
despite inherent biases and variability.

\begin{figure}[h]
\centering
\includegraphics[width=0.32\textwidth]{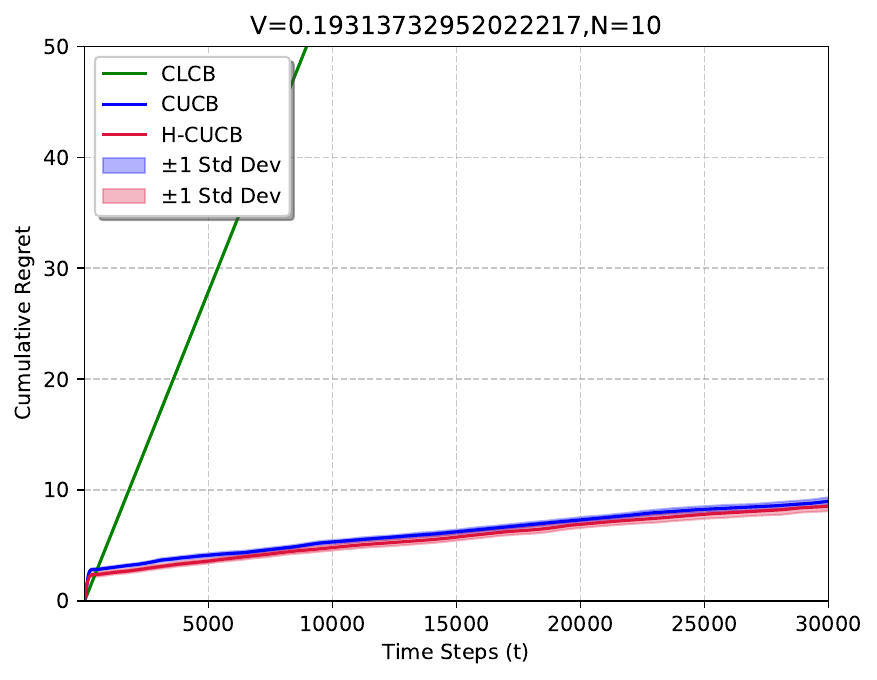}
   \includegraphics[width=0.32\textwidth]{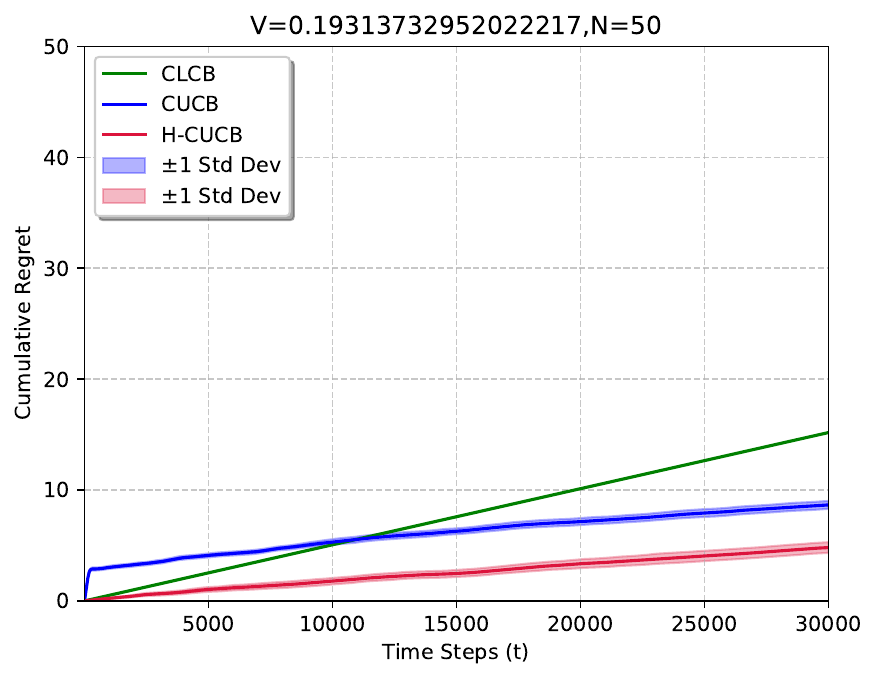}
   \includegraphics[width=0.32\textwidth]{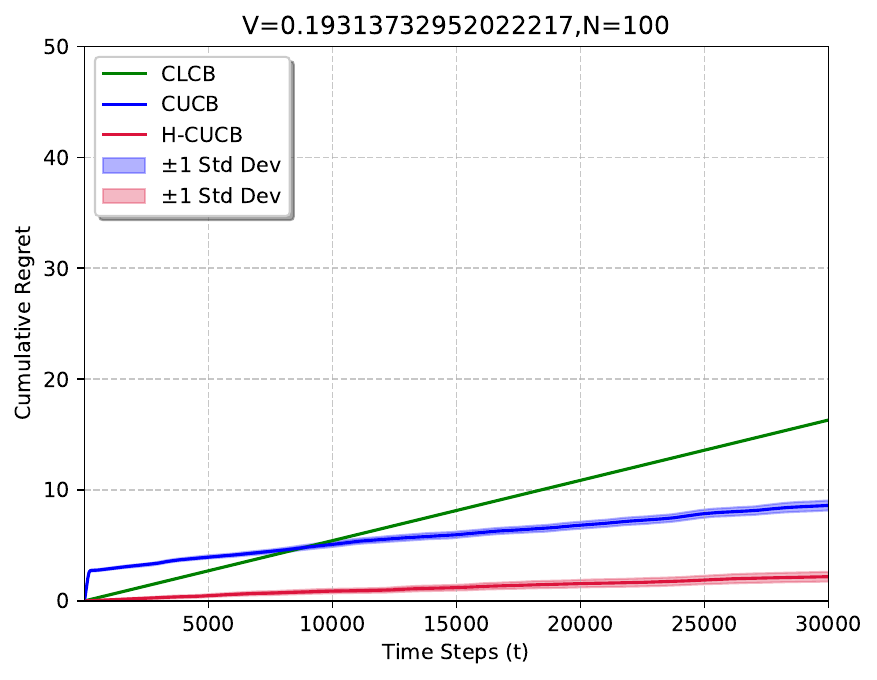}
    \caption{Performance comparison of hybrid CUCB against baselines in a real-world dataset. }
    \label{fig:exp_bias}
\end{figure}

\section{Lower bound of hybrid CMAB-T}
\subsection*{Step 1. Problem Setup.}

We consider a CMAB-T instance with $m$ base arms.  
An action (super-arm) is a subset $S \subseteq [m]$ with fixed size $|S| = K$.

\medskip
\noindent\textbf{Triggering distribution.}
When action $S$ is played at round $t$, exactly one arm in $S$ is triggered:
\[
p_{S,i} = \Pr(i \in \tau_t \mid S_t = S) = 
\begin{cases}
1/K, & i \in S, \\
0,   & i \notin S.
\end{cases}
\]
Let $T_i$ denote the number of rounds in which arm $i$ is triggered.

\medskip
\noindent\textbf{Reward model.}
When arm $i$ is triggered, its outcome is $X_i^{(t)} \sim \mathrm{Bern}(\mu_i)$.
The reward is defined as
\[
R(S, X, \tau) = B\cdot \sum_{i\in\tau} X_i
            = B\cdot X_{j^\ast},
\]
where $j^\ast$ is the unique triggered arm.

Thus the expected reward of action $S$ under mean vector $\mu$ is
\[
r_S(\mu)
    = \mathbb{E}[R(S,X,\tau)]
    = B \sum_{j \in S} p_{S,j} \mu_j
    = \frac{B}{K} \sum_{j\in S} \mu_j.
\]

\medskip
\noindent\textbf{Verification of TPM smoothness.}
For any two mean vectors $\mu,\mu'$,
\[
\begin{aligned}
|r_S(\mu) - r_S(\mu')|
    &= \left| \frac{B}{K}\sum_{j\in S}(\mu_j - \mu_j') \right| \\
    &\le \frac{B}{K} \sum_{j\in S} |\mu_j - \mu_j'| 
     = B \sum_{j} p_{S,j} |\mu_j - \mu_j'|.
\end{aligned}
\]
Therefore this instance satisfies the TPM condition with constant $B$.

\medskip
\noindent\textbf{Gaps.}
Assume the optimal arm is arm $1$, and its mean dominates:
\[
\mu_1 = \mu_2 = \cdots = \mu_K \ge \mu_{K+1} \ge \cdots \ge \mu_m.
\]
The optimal action is $S^{\*} = \{1,\dots,K\}$ with reward
\[
r_{S^{\*}}(\mu) = \frac{B}{K} \cdot K \mu_1 = B\mu_1.
\]

For any suboptimal arm $i > K$, define an action
\[
S_i = \{1,1,\dots,1,i\} \qquad (K-1 \text{ copies of arm } 1,\text{ and } i).
\]

(Since $\mu_1 = \mu_2 = \cdots = \mu_K$)

Its expected reward is
\[
r_{S_i}(\mu)
= \frac{B}{K}\left((K-1)\mu_1 + \mu_i\right).
\]

Thus the action gap is
\[
\Delta_{S_i}
= r_{S^{\*}}(\mu) - r_{S_i}(\mu)
= \frac{B}{K}\,(\mu_1 - \mu_i)
= \frac{B}{K}\Delta_i,
\]
where we define the base-arm gap $\Delta_i := \mu_1 - \mu_i$.

\medskip
\noindent\textbf{Regret decomposition.}
Let $N_i^{\text{act}}$ be the number of times the algorithm selects action $S_i$.
Then the regret contributed by arm $i$ is
\[
\mathrm{Reg}_i(T)
= \Delta_{S_i} \cdot \mathbb{E}_\nu[N_i^{\text{act}}]
= \frac{B}{K}\Delta_i\,\mathbb{E}_\nu[N_i^{\text{act}}].
\]

Because arm $i$ is triggered with probability $1/K$ under $S_i$, we have
\[
\mathbb{E}_\nu[T_i]
= \frac{1}{K}\, \mathbb{E}_\nu[N_i^{\text{act}}]
\quad\Leftrightarrow\quad
\mathbb{E}_\nu[N_i^{\text{act}}]
= K\cdot \mathbb{E}_\nu[T_i].
\]

Thus,
\[
\boxed{
\mathrm{Reg}_i(T)
= B\Delta_i \cdot \mathbb{E}_\nu[T_i].
}
\]

This establishes the link between regret and the required online triggering counts.

\subsection*{Step 2. Construct Two Instances.}

For arm $i$, we consider two environments $\nu$ and $\nu^{(i)}$:

\[
\text{online:}\quad 
\mu_i^{\mathrm{on}} \;\mapsto\; \mu_i^{\mathrm{on}} + \Delta,
\qquad
\text{and for all } j\neq i,\;\;
\mu_j^{\mathrm{on}} \text{ unchanged}.
\]

\[
\text{offline:}\quad
\begin{cases}
\Delta \le 2V_i: & \text{the offline means can fully align},\\[4pt]
\Delta > 2V_i: & \text{match as closely as possible, but cannot fully align}.
\end{cases}
\]

Since,
\[
\mu_i^{\mathrm{off},(i)} \ge \mu_i^{\mathrm{on}} + \Delta - V_i,
\qquad
\mu_i^{\mathrm{off}} \le \mu_i^{\mathrm{on}} + V_i.
\]

Hence the difference cannot be smaller than
\[
\mu_i^{\mathrm{off},(i)} - \mu_i^{\mathrm{off}}
\;=\;
\Delta - 2V_i,
\quad\text{i.e.,}\quad
\mu_i^{\mathrm{off}} = \mu_i^{\mathrm{on}} + V_i,
\qquad
\mu_i^{\mathrm{off},(i)} = \mu_i^{\mathrm{on}} + \Delta - V_i.
\]

Trigger-$i$ outcomes are Bernoulli$(\mu_i)$ in both online and offline data
(with different means under the two environments).

\medskip
\noindent\textbf{KL terms.}

\[
\mathrm{KL}\bigl(P_{\nu}^{\mathrm{on}} \,\Vert\, P_{\nu'}^{\mathrm{on}} \bigr)
= 
\mathbb{E}_{\nu}[T_i]\cdot 
\mathrm{KL}\!\left(
   \mathrm{Bern}(\mu_i^{\mathrm{on}})\,\Vert\,
   \mathrm{Bern}(\mu_i^{\mathrm{on}} + \Delta)
\right).
\]

\[
\mathrm{KL}\bigl(P_{\nu}^{\mathrm{off}} \,\Vert\, P_{\nu'}^{\mathrm{off}} \bigr)
=
N_i\cdot 
\mathrm{KL}\!\left(
   \mathrm{Bern}(\mu_i^{\mathrm{off}})\,\Vert\,
   \mathrm{Bern}(\mu_i^{\mathrm{off},(i)})
\right).
\]

By standard bandit lower bounds or the Bretagnolle--Huber inequality, we have
\[
\mathbb{E}_{\nu}[T_i]\cdot
\mathrm{KL}\!\left(
\mathrm{Bern}(\mu_i^{\mathrm{on}})\,\Vert\,
\mathrm{Bern}(\mu_i^{\mathrm{on}}+\Delta)
\right)
\;+\;
N_i\cdot 
\mathrm{KL}\!\left(
\mathrm{Bern}(\mu_i^{\mathrm{off}})\,\Vert\,
\mathrm{Bern}(\mu_i^{\mathrm{off},(i)})
\right)
\;\gtrsim\;
\log T.
\]

Since
\[
\mathrm{KL}\!\left(\mathrm{Bern}(\mu)\,\Vert\,\mathrm{Bern}(\mu+\Delta)\right)
= \Theta(\Delta^{2})
\quad
(\text{from~}
\Delta^{2}/2\mu(1-\mu)
\;\le\;
\mathrm{KL}
\;\le\;
\Delta^{2}/\mu(1-\mu)),
\]
we obtain
\[
\mathbb{E}_{\nu}[T_i]\cdot \Delta^{2}
\;+\;
N_i\cdot (\Delta - 2V_i)^{2}
\;\gtrsim\;
\log T.
\]

Therefore,
\[
\mathbb{E}_{\nu}[T_i]
\;\gtrsim\;
\frac{\log T}{\Delta^{2}}
\;-\;
N_i\cdot \left(\frac{\Delta - 2V_i}{\Delta}\right)^{2}
:=
\frac{\log T}{\Delta^{2}}
-
N_i^{\prime\prime},
\]
where
\[
N_i''=N_i\min\{1-\frac{2V_i}{\Delta_i},0\}^2 .
\]

\subsection*{Step3. Summary}

Finally, we have:

\begin{align*}
\mathrm{Reg}_i(T)
&\ge B \cdot \Delta_i \cdot \mathbb{E}_{\nu}[T_i] \\
&\ge B\left( 
        \frac{\log T}{\Delta_i}
        - N_i^{\prime\prime} \cdot \Delta_i
     \right)
     \qquad (\text{Setting} ~\Delta = \Delta_i) \\
&\ge
   B\left(
      \frac{\log T}{\Delta_i}
      - \sqrt{N_i^{\prime\prime} \cdot \log T}
    \right),
     \qquad 
     \left(
        N_i^{\prime\prime} < \frac{\log T}{\Delta_i^2}
        \;\Rightarrow\;
        \Delta_i = \Delta 
        < \sqrt{\frac{\log T}{N_i^{\prime\prime}}}
     \right).
\end{align*}

Notice that $\Delta_i=gap= \frac{B}{K} \Delta_{\min}^i$, and $\omega_i=2V_i$, then we have:

\begin{align*}
\mathrm{Reg}(T)
= \sum_{i \in [m]} \mathrm{Reg}_i(T) 
&\ge
   B\sum_{i \in[m]}\left(
      \frac{\log T}{\Delta_i}
      - \sqrt{N_i^{\prime\prime} \cdot \log T}
    \right) \\
&\ge
   B\sum_{i \in[m]}\left(
      \frac{K\log T}{B\Delta^i_{\min}}
      - \sqrt{N_i^{\prime\prime} \cdot \log T}
    \right), \text{where} ~ N_i''=N_i\min\{1-\frac{B\omega_i}{K\Delta^i_{\min}},0\}^2.
\end{align*}

% \section*{Use of LLMs}
% During the preparation of this paper, large language models (LLMs) such as ChatGPT were used 
% for minor assistance in two ways: 
% (i) language editing and polishing of the manuscript, and 
% (ii) providing quick code documentation lookup and boilerplate coding assistance for experiments. 
% All technical contributions, algorithm design, theoretical analysis, and experimental results were 
% developed and verified by the authors. 
% The use of LLMs did not affect the originality or the scientific validity of this work.

\end{document}